\documentclass[11pt]{article} 
\usepackage{url}
\usepackage{smile}
\usepackage{graphicx} 
\usepackage{algorithm}
\usepackage{algorithmic}
\usepackage{todonotes}
\usepackage{epstopdf}
\usepackage{wrapfig}
\usepackage[colorlinks, linkcolor=blue, anchorcolor=blue, citecolor=blue]{hyperref}
\usepackage[margin=1in]{geometry}
\usepackage[normalem]{ulem}
\usepackage[export]{adjustbox}
\usepackage{mathtools, cuted}
\usepackage{natbib}
\usepackage{enumerate}
\usepackage{microtype}
\usepackage{graphicx}
\usepackage{booktabs} 
\usepackage{smile}
\usepackage{wrapfig}

\usepackage{kpfonts}
\DeclareMathAlphabet{\mathsf}{OT1}{cmss}{m}{n}
\SetMathAlphabet{\mathsf}{bold}{OT1}{cmss}{bx}{n}

\begin{document}

\title{\huge A Fast Proximal Point Method for Computing Exact Wasserstein Distance}

\author{Yujia Xie, Xiangfeng Wang, Ruijia Wang, Hongyuan Zha \thanks{Yujia Xie, Ruijia Wang and Hongyuan Zha  are affiliated with College of Computing at Georgia Institute of Technology. Xiangfeng Wang is affiliated with East China Normal University. Emails: {\tt Xie.Yujia000@gmail.com, xfwang@sei.ecnu.edu.cn, rwang@gatech.edu, zha@cc.gatech.edu}.}}
\date{}

\maketitle

\begin{abstract}


Wasserstein distance plays increasingly important roles in machine learning, stochastic programming and image processing. 
Major efforts have been under way to address its high computational complexity, some leading to approximate or regularized variations such as Sinkhorn distance. 
However, as we will demonstrate, regularized variations with large regularization parameter will degradate the performance in several important machine learning applications, and small regularization parameter will fail due to numerical stability issues with existing algorithms. 
We address this challenge by developing an Inexact Proximal point method for exact Optimal Transport problem (IPOT) with the proximal operator approximately evaluated at each iteration using projections to the probability simplex. 
The algorithm (a) converges to exact Wasserstein distance with theoretical guarantee and robust regularization parameter selection, (b) alleviates numerical stability issue, (c) has similar computational complexity to Sinkhorn, and (d) avoids the shrinking problem when apply to generative models.
Furthermore, a new algorithm is proposed based on IPOT to obtain sharper Wasserstein barycenter.

\end{abstract}

\vspace{-5pt}
\section{INTRODUCTION}
\label{section1}

\vspace{-5pt}

Many practical tasks in machine learning rely on computing a Wasserstein distance between probability measures or between their sample points~\citep{arjovsky2017wasserstein,thorpe2017transportation,ye2017fast,srivastava2015wasp}. However, the high computational cost of Wasserstein distance has been a thorny issue and has limited its application to challenging machine learning problems.

In this paper we focus on Wasserstein distance for discrete distributions the computation of which amounts to solving the following discrete \textit{optimal transport} (OT) problem, 
\begin{eqnarray}\label{eq:target1}
\begin{aligned}
W(\bm{\mu},\bm{\nu}) = \sideset{}{_{\bm{\Gamma}\in\Sigma(\bm{\mu},\bm{\nu})}}\min \langle\bm{C},\bm{\Gamma}\rangle.
\end{aligned}
\end{eqnarray}
Here $\bm{\mu},\bm{\nu}$ are two probability vectors, $W(\bm{\mu},\bm{\nu})$ is the Wasserstein distance between $\bm{\mu}$ and $\bm{\nu}$. 
Matrix $\bm{C}=[c_{ij}]\in\mathbb{R}_{+}^{m\times n}$ is the \textit{cost matrix}, whose element $c_{ij}$ represents the distance between the $i$-th support point of $\bm{\mu}$ and the $j$-th one of $\bm{\nu}$. The optimal solution $\bm{\Gamma}^*$ is referred as \textit{optimal transport plan}.
Notation $\langle\cdot,\cdot\rangle$ represents the Frobenius dot-product and $\Sigma(\bm{\mu},\bm{\nu})=\{\bm{\Gamma}\in\mathbb{R}_{+}^{m\times n}: \bm{\Gamma}\bm{1}_m=\bm{\mu}, \bm{\Gamma}^{\top}\bm{1}_n=\bm{\nu}\}$, where $\bm{1}_n$ represents $n$-dimensional vector of ones. This is a linear programming problem with typical super $O(n^3)$ complexity\footnote{Assume $O(n)=O(m)$.}.

An effort by Cuturi to reduce the complexity leads to a regularized variation of  (\ref{eq:target1}) giving rise the so-called Sinkhorn distance~\citep{cuturi2013sinkhorn}. It  aims to solve an entropy regularized optimal transport problem
\begin{eqnarray}\label{eq:sinkhorn}
\begin{aligned}
W_{\epsilon}(\bm{\mu},\bm{\nu}) = \sideset{}{_{\bm{\Gamma}\in\Sigma(\bm{\mu},\bm{\nu})}}\min \langle\bm{C},\bm{\Gamma}\rangle+\epsilon h(\bm{\Gamma}).
\end{aligned}
\end{eqnarray}
The entropic regularizer $h(\bm{\Gamma})=\sum_{i,j}\Gamma_{ij}\ln\Gamma_{ij}$ results in an 
optimization problem (\ref{eq:sinkhorn}) that can be solved efficiently by iterative Bregman projections~\citep{benamou2015iterative},
 \[
 \bm{a}^{(l+1)} = \frac{\pmb{\mu}}{\bm{Gb}^{(l)}}, \quad \bm{b}^{(l+1)} = \frac{\pmb{\nu}}{\bm{G}^T \bm{a}^{(l+1)}}
 \]
starting from $b^{(0)} = \frac{1}{n}\pmb{1}_n$, where $\bm{G}=[G_{ij}]$ and $ G_{ij} = e^{-C_{ij}/\epsilon}$. The optimal solution $\bm{\Gamma}^*$ then takes the form $\Gamma^*_{ij} = a_i G_{ij} b_j$.
The iteration is also referred as \textit{Sinkhorn iteration}, and the method is referred as \textit{Sinkhorn algorithm} which, recently, is proven to achieve a  near-$O(n^2)$ complexity ~\citep{altschuler2017near}.

The choice of $\epsilon$ cannot be arbitrarily small. 
Firstly, $ G_{ij} = e^{-C_{ij}/\epsilon}$ tends to underflow if $\epsilon$ is very small.
The methods in~\citet{benamou2015iterative,chizat2016scaling,mandad2017variance} try to address this numerical instability by performing the computation in log-space, but they require a significant amount of extra exponential and logarithmic operations, and thus,  compromise the advantage of efficiency.
More significantly, even with the benefits of log-space computation, the linear convergence rate of the Sinkhorn algorithm is determined by the \textit{contraction ratio} $\kappa(\bm{G})$, which approaches $1$ as $\epsilon \to 0$~\citep{franklin1989scaling}.
Consequently, we observe drastically increased number of iterations for Sinkhorn method when using small $\epsilon$.

Can we just employ Sinkhorn distance with a moderately sized $\epsilon$ for machine learning problems so that we can get the benefits of the reduced complexity? Some applications show Sinkhorn distance can generate good results with a moderately sized $\epsilon$~\citep{genevay2017sinkhorn,martinez2016relaxed}. However, we show that in several important problems such as generative model learning and Wasserstein barycenter computation, a moderately sized $\epsilon$ will significantly degrade the performance while the Sinkhorn algorithm with a very small $\epsilon$ becomes prohibitively expensive (also shown in \citet{solomon2014wasserstein}). 

In this paper, we propose a new framework, Inexact Proximal point method for Optimal Transport (IPOT) to compute the Wasserstein distance using generalized proximal point iterations based on Bregman divergence. To enhance efficiency, the proximal operator is inexactly evaluated at each iteration using projections to the probability simplex, leading to an \textbf{inexact} update at each iteration yet converging to the \textbf{exact} optimal transport solution.

Regarding the theoretical analysis of IPOT, we provide conditions on the number of inner iterations that will guarantee the linear convergence of IPOT. In fact, empirically, IPOT behaves better than the analysis: the algorithm seems to be linearly convergent with just one inner iteration, demonstrating its efficiency. We also perform several other tests to show the excellent performance of IPOT. As we will discussed in Section \ref{section42}, the computation complexity is almost indistinguishable comparing to the Sinkhorn method. Yet again, IPOT avoids the lengthy and experience-based tuning of the $\epsilon$ and can converges to the true optimal transport solution robustly with respect to its own parameters. This is unquestionably important in applications where the exact sparse transport plan is preferred. In applications where only Wasserstein distance is needed, the bias caused by regularization might also be problematic. As an example, when applying Sinkhorn to generative model learning, it causes the shrinkage of the learned distribution towards the mean, and therefore cannot cover the whole support of the target distribution adequately. 

Furthermore, we develop another new algorithm based on the proposed IPOT to compute Wasserstein barycenter (see Section \ref{section5}). Better performance is obtained with much sharper images. It turns out that the inexact evaluation of the proximal operator blends well with Sinkhorn-like barycenter iteration.

\vspace{-5pt}
\section{PRELIMINARIES}
\label{section2}

\vspace{-5pt}

We then provide some background on optimal transport and proximal point method.


\subsection{Wasserstein Distance and Optimal Transport}
\label{section21}

\vspace{-5pt}

Wasserstein distance is a metric for two probability distributions. 
Given two distributions $\mu$ and $\nu$, the $p$-Wasserstein distance between them is defined as
\begin{eqnarray}\label{eq:wd_c}
\begin{aligned}
W_p(\mu,\nu):=\Big\{\inf_{\gamma\in\Sigma(\mu,\nu)}\int_{\mathcal{M}\times\mathcal{M}}d^p(x,y)\text{d}\gamma(x,y)\Big\}^{\frac{1}{p}},
\end{aligned}
\end{eqnarray}
where $\Sigma(\mu,\nu)$ is the set of joint distributions whose marginals are $\mu$ and $\nu$, respectively. The above optimization problem is also called the Monge-Kantorovitch problem or \textit{optimal transport} problem~\citep{kantorovich1942mass}. In the following, we focus on the $2$-Wasserstein distance, and for convenience we write $W(\cdot,\cdot)=W_2^2(\cdot,\cdot)$. 


When $\mu$ and $\nu$ both have finite supports, we can represent the distributions as vectors $\bm{\mu}\in\mathbb{R}_{+}^{m},\bm{\nu}\in\mathbb{R}_{+}^{n}$, where $\|\bm{\mu}\|_1=\|\bm{\nu}\|_1=1$. Then the Wasserstein distance between $\bm{\mu}$ and $\bm{\nu}$ is computed by (\ref{eq:target1}).
In other cases, given realizations $\{x_i\}_{i=1}^{m}$ and $\{y_j\}_{j=1}^{n}$ of $\mu$ and $\nu$, respectively, we can approximate them by empirical distributions $\hat{\mu}=\frac{1}{m}\sum_{x_i}\delta_{x_i}$ and $\hat{\nu}=\frac{1}{n}\sum_{y_j}\delta_{y_j}$. 
The supports of $\hat{\mu}$ and $\hat{\nu}$ are finite, so similarly we have $\bm{\mu}=\frac{1}{m}\bm{1}_{\{x_i\}}$, $\bm{\nu}=\frac{1}{n}\bm{1}_{\{y_j\}}$, and $\bm{C}=[c(x_i, y_j)]\in \mathbb{R}_{+}^{m\times n}$. 

The optimization problem (\ref{eq:target1}) can be solved by linear programming (LP) methods. 
LP tends to provide a sparse solution, which is preferable in applications like histogram calibration or color transferring~\citep{rabin2014adaptive}. However, the cost of LP scales at least $O(n^3 \log n)$ for general metric, where $n$ is the number of data points~\citep{pele2009fast}. 
As aforementioned, an alternative optimization method is the Sinkhorn algorithm in~\cite{cuturi2013sinkhorn}.
Following the same strategy, many variants of the Sinkhorn algorithm have been proposed~\citep{altschuler2017near, dvurechensky2017adaptive,thibault2017overrelaxed}. 
Unfortunately, all these methods only approximate original optimal transport by its regularized version and their performance both in terms of numerical stability and computational complexity is sensitive to the choice of $\epsilon$. 

\begin{figure*}[!t]
\centering
\includegraphics[width=0.97\textwidth]{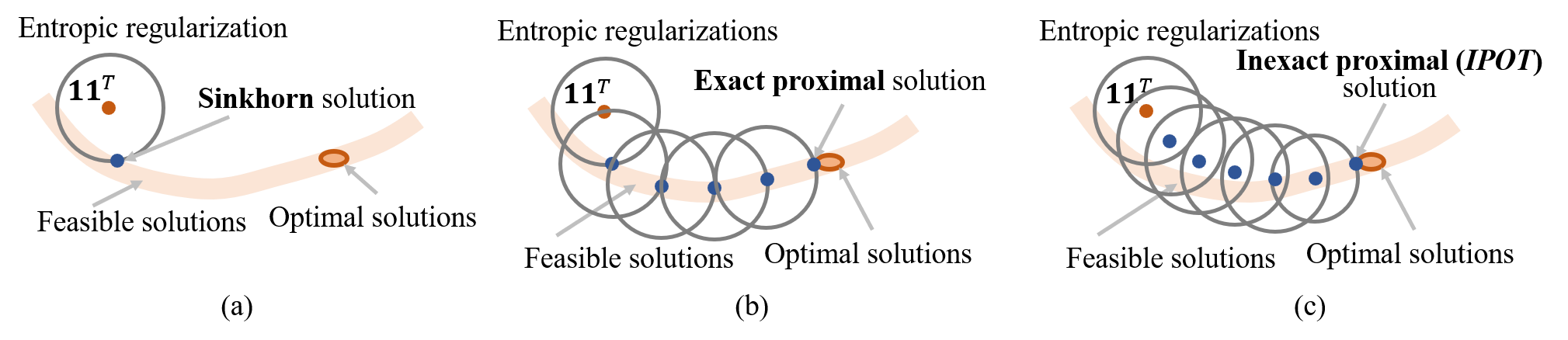}
\vspace{-0.15in}
\caption{\label{fig:sketch} Schematic of the convergence path of (a) Sinkhorn algorithm, (b) exact proximal point algorithm and (c) inexact proximal point algorithm (IPOT). The distance shown is in Bregman sense. Sinkhorn solution is feasible and the closest to optimal solution set within the $D_h$ constraints, but is not in the optimal solution set. However, proximal point algorithm, no matter exact or inexact, solves optimization with $D_h$ constraints iteratively, until an optimal solution is reached.}
\end{figure*}

\vspace{-5pt}

\subsection{Generalized Proximal Point Method}
\vspace{-5pt}

Proximal point methods are widely used in optimization~\citep{afriat1971theory,parikh2014proximal,rockafellar1976augmented,rockafellar1976monotone}. Here, we introduce its generalized form.
Given a convex objective function $f$ defined on $\mathcal{X}$ with optimal solution set $\mathcal{X}^*\subset \mathcal{X}$, generalized proximal point algorithm aims to solve
\begin{eqnarray}\label{eq:minimum}
\begin{aligned}
\sideset{}{_{x\in \mathcal{X}}}\argmin f(x).
\end{aligned}
\end{eqnarray} 

In order to solve Problem (\ref{eq:minimum}), the algorithm generates a sequence $\{x^{(t)}\}_{t=1,2,...}$ by the following generalized proximal point iterations: 
\begin{eqnarray}\label{eq:proximal}
\begin{aligned}
x^{(t+1)} = \sideset{}{_{x\in \mathcal{X}}}\argmin f(x) + \beta^{(t)} d(x, x^{(t)}),
\end{aligned}
\end{eqnarray} 
where $d$ is a regularization term used to define the proximal operator, usually defined to be a closed proper convex function. For classical proximal point method, $d$ adopts the square of Euclidean distance, i.e., $d(x,y) =\|x-y\|_2^2$, in which case the sequence $\{x^{(t)}\}$ converges to an element in $\mathcal{X}^*$ almost surely.

The generalized proximal point method has many advantages, e.g, it has a robust convergence behavior 
-- a fairly mild condition on $\beta^{(t)}$ guarantee its convergence for some given $d$, and the 
specific choice of $\beta^{(t)}$ generally just affects its convergence rate. 
Moreover, even if the proximal operator defined in (\ref{eq:proximal}) is not exactly evaluated in each iteration, giving rise to inexact proximal point methods, the global convergence of which with local linear rate is still guaranteed under certain conditions~\citep{solodov2001unified,schmidt2011convergence}. 


\vspace{-5pt}

\section{BREGMAN DIVERGENCE BASED  PROXIMAL POINT METHOD}
\vspace{-5pt}

\label{section3}
In this section we will develop the main algorithm IPOT.
Specifically, we will use generalized proximal point method to solve the optimal transport problem \eqref{eq:target1}.  Recall the proximal point iteration \eqref{eq:proximal}, we take $f(\bm{\Gamma})=\langle \bm{C}, \bm{\Gamma} \rangle$, $\mathcal{X}=\Sigma(\bm{\mu},\bm{\nu})$, and $d(\cdot,\cdot)$ 
to be Bregman divergence $D_h$ based on entropy function $h(\bm{x})=\sum_i x_i \ln x_i$, i.e., 
\begin{eqnarray}\label{bregman_divergence}
\begin{aligned}
D_h (\bm{x},\bm{y}) = \sum_{i=1}^n x_i \log \frac{x_i}{y_i} - \sum_{i=1}^n x_i + \sum_{i=1}^n y_i.
\end{aligned}
\end{eqnarray}
As a result, the proximal point iteration for problem \eqref{eq:target1} can be written as
\begin{align}\label{proximal2}
\bm{\Gamma}^{(t+1)} = \sideset{}{_{\bm{\Gamma}\in\Sigma(\bm{\mu},\bm{\nu})}}\argmin\langle \bm{C}, \bm{\Gamma} \rangle + \beta^{(t)} D_h(\bm{\Gamma},\bm{\Gamma}^{(t)}).
\end{align}
However, it is still not trivial to solve the above optimization problem in each iteration, since optimization problems with such complicated constraints generally does not have a closed-form solution. Fortunately, with some reorganization, we can solve it with Sinkhorn algorithm. 

Substituting Bregman divergence \eqref{bregman_divergence} into proximal point iteration (\ref{proximal2}), with simplex constraints, we obtain
\begin{eqnarray}\label{proximal3}
\begin{aligned}
\bm{\Gamma}^{(t+1)}  
=  \sideset{}{_{\bm{\Gamma}\in\Sigma(\bm{\mu},\bm{\nu})}}\argmin  \langle \bm{C}-\beta^{(t)} \log \bm{\Gamma}^{(t)}, \bm{\Gamma} \rangle + \beta^{(t)}  h(\bm{\Gamma}).  
\end{aligned}
\end{eqnarray}
Denote $\bm{C}'=\bm{C}-\beta^{(t)} \ln\bm{\Gamma}^{(t)}$. Note that for optimization problem \eqref{proximal3}, $\bm{\Gamma}^{(t)}$ is a fixed value that is not relevant to optimization variable $\bm{\Gamma}$. Therefore, $\bm{C}'$ can be viewed as a new cost matrix that is known, and problem \eqref{proximal3} is an entropy regularized optimal transport problem. 
Comparing to \eqref{eq:sinkhorn}, problem \eqref{proximal3} can be solved by Sinkhorn iteration by replacing $G_{ij}$ by $G'_{ij} = e^{-C'_{ij}/\beta^{(t)}} = \Gamma_{ij}^{(t)} e^{-C_{ij}/\beta^{(t)}}$. As we will later shown in Section \ref{section32}, as $t\to \infty$, $\bm{\Gamma}^{(t)}$ will converge to an optimal transport plan.

     \begin{algorithm}[t]
       \caption{IPOT($\pmb{\mu},\pmb{\nu},\bm{C}$)}
	\label{algo1}       
	 \begin{algorithmic}
          \STATE \textbf{Input:} Probabilities $\{\pmb{\mu},\pmb{\nu}\}$ on support points $\{x_i\}_{i=1}^m$, $\{y_j\}_{j=1}^n$, cost matrix $\bm{C}=[\|x_i-y_j\|]$
	\STATE $\bm{b}\leftarrow \frac{1}{m}\pmb{1}_m$
	\STATE $G_{ij}\leftarrow e^{-\frac{C_{ij}}{\beta}}$
	\STATE $\bm{\Gamma}^{(1)} \leftarrow  \pmb{11}^T$
	\FOR {$t=1,2,3,...$}
	\STATE $\bm{Q} \leftarrow  \bm{G}\odot \bm{\Gamma}^{(t)}$
	\STATE \textbf{for} {$l=1,2,3,...,L$} \textbf{do}$\quad$ // Usually set $L=1$
	\STATE \quad $\bm{a}\leftarrow \frac{\pmb{\mu}}{\bm{Qb}}$, $\bm{b}\leftarrow \frac{\pmb{\nu}}{\bm{Q}^T \bm{a}}$
	\STATE \textbf{end for}
	\STATE $\bm{\Gamma}^{(t+1)} \leftarrow  \text{diag}(\bm{a}) \bm{Q} \text{diag}(\bm{b})$
	\ENDFOR
        \end{algorithmic}
      \end{algorithm}

Figure \ref{fig:sketch} illustrates how Sinkhorn and IPOT solutions approach optimal solution with respect to number of iterations in sense of Bregman divergence. First, let's consider Sinkhorn algorithm. The objective function of Sinkhorn \eqref{eq:sinkhorn} has regularization term $\epsilon h(\bm{\Gamma})$, which can be equivalently rewritten as constraint $D_h(\bm{\Gamma},\bm{11}^T)\leq \eta$ for some $\eta>0$. Therefore, Sinkhorn solution is feasible within the $D_h$ constraints and the closest to optimal solution set, as shown in Figure \ref{fig:sketch} (a). 

Proximal point algorithms, on the other hand, solves optimization with $D_h$ constraints iteratively as shown in Figure \ref{fig:sketch} (b)(c). Different from Sinkhorn algorithm, proximal point algorithms converge to the optimal solution with nested iterative loops.
Exact proximal point method, i.e., solving (\ref{proximal3}) exactly as shown in Figure \ref{fig:sketch} (b),  provides a feasible solution that is closest to the optimal solution set in each proximal iteration until the optimal solution reached. However, the disadvantage for exact proximal point method is that it's not efficient. 

The proposed inexact proximal point method (IPOT) does not solve (\ref{proximal3}) exactly. Instead, a very small amount of Sinkhorn iteration, e.g., only one iteration, is suggested. The reason for this is three-fold.  First, the convergence of Sinkhorn algorithm in each proximal iteration is not required, since it is just intermediate step. Second, usually in numerical optimization, the first a few iterations achieve the most decreasing in the objective function. Performing only the first a few iterations has high cost performance. Last and perhaps the most important, it is observed that IPOT can still converge to an exact solution with small amount of inner iterations\footnote{Similar ideas are also used in accelerating the expectation-maximization algorithm with only one iteration used in the maximization step  \citep{lange1995gradient}.}.

The algorithm is shown in Algorithm \ref{algo1}. For simplicity we use $\beta = \beta^{(t)}$. Denote $\text{diag}(\bm{a})$ the diagonal matrix with $a_i$ as its $i$-th diagonal elements. Denote $\odot$ as element-wise matrix multiplication and $\frac{(\cdot)}{(\cdot)}$ as element-wise division. We use warm start to improve the efficiency, i.e. in each proximal point iteration, we use the final value of $\bm{a}$ and $\bm{b}$ from last proximal point iteration as initialization instead of $\bm{b}^{(0)} = \pmb{1}_m$.
Later we will show empirically IPOT will converge under a large range of $\beta$ with $L=1$,
a single inner iteration will suffice.

\vspace{-5pt}
\section{WASSERSTEIN BARYCENTER BY IPOT}
\label{section5}
\vspace{-5pt}

We now extend IPOT method to a related problem -- computing the Wasserstein barycenter. Wasserstein barycenter is widely used in machine learning and computer vision \citep{benamou2015iterative, rabin2011wasserstein}.
Given a set of distributions $\mathcal{P} = \{\bm{p}_1,\bm{p}_2,...,\bm{p}_K\}$, their Wasserstein barycenter is defined as
\begin{eqnarray}\label{big_target}
\begin{aligned}
\bm{q}^*(\mathcal{P},\bm{\lambda}) = \sideset{}{_{\bm{q}\in\mathcal{Q}}}\argmin \sideset{}{_{k=1}^K}\sum \lambda_k W(\bm{q},\bm{p}_k)
\end{aligned}
\end{eqnarray}
where $\mathcal{Q}$ is in the space of probability distributions, $\sum_{k=1}^K \lambda_k = 1$, and $W(\bm{q},\bm{p}_k)$ is the Wasserstein distance between the barycenter $\bm{q}$ and distribution $\bm{p}_k$, which takes the form
\begin{equation}\label{eq:sampleW}
W(\bm{q},\bm{p}_k) = \min_{\bm{\Gamma}} {\langle \bm{C}, \bm{\Gamma} \rangle },\quad \text{s.t.}\quad \bm{\Gamma}\textbf{1}=\bm{p}_k, \bm{\Gamma}^T\textbf{1}=\bm{q}.
\end{equation}

The idea of IPOT method can also be used to compute Wasserstein barycenter. 
Substitute (\ref{eq:sampleW}) into \eqref{big_target} and reorganize, we have
\begin{align*}
\bm{q}^*(\mathcal{P},\bm{\lambda}) = \sideset{}{_{\bm{q}\in\mathcal{Q}}}\argmin  \sum_{k=1}^K \lambda_k {\langle \bm{C},\bm{\Gamma}_k \rangle }, \quad \text{s.t.} \quad \bm{\Gamma}_k\textbf{1}=\bm{p}_k,\text{ and } \exists \bm{q},\bm{\Gamma}_k^T\textbf{1}=\bm{q}.
\end{align*}
Analogous to IPOT, we take 
\begin{align*}
    f(\{\bm{\Gamma}_k\})=\sum_{k=1}^K \lambda_k {\langle \bm{C}, \bm{\Gamma}_k \rangle },
\end{align*}
take $\mathcal{X}$ to be the corresponding constraints, and take $d(\{\bm{\Gamma}_k\},\{\bm{\Gamma}_k^{(t)}\})$ to be $\sum_{k=1}^K \lambda_k D_h(\bm{\Gamma}_k,\bm{\Gamma}_k^{(t)})$. The proximal point iteration for barycenter is
\begin{align}
\bm{\Gamma}_k^{(t+1)} = \argmin_{\bm{\Gamma}_k} \sum_{k=1}^K \lambda_k \left( \langle \bm{C},\bm{\Gamma}_k \rangle +\beta^{(t)} D_h(\bm{\Gamma}_k,\bm{\Gamma}_k^{(t)})\right) \quad
\text{s.t.}\quad \bm{\Gamma}_k\textbf{1}=\bm{p}_k, \text{ and } \exists \bm{q},\bm{\Gamma}_k^T\textbf{1}=\bm{q}. \label{eq:bw_iter1}
\end{align}
With further organization, we have
\begin{align}
\bm{\Gamma}_k^{(t+1)} = \argmin_{\bm{\Gamma}_k}  \sum_{k=1}^K \lambda_k \left( \langle  \bm{C}-\beta^{(t)} \log \bm{\Gamma}_k^{(t)},\bm{\Gamma}_k \rangle  +\beta^{(t)} h(\bm{\Gamma}_k)\right)
\quad \text{s.t.}\quad \bm{\Gamma}_k\textbf{1}=\bm{p}_k, \text{ and } \exists \bm{q},\bm{\Gamma}_k^T\textbf{1}=\bm{q}. \label{eq:bw_iter}
\end{align}
On the other hand, analogous to Sinkhorn algorithm, \citet{benamou2015iterative} propose \textit{Bregman iterative projection} that seeks to solve an entropy regularized barycenter,
\begin{eqnarray}\label{eq:sink-wb}
\begin{aligned}
\bm{q}_{\epsilon}^*(\mathcal{P},\bm{\lambda}) = \sideset{}{_{\bm{q}\in\mathcal{Q}}}\argmin \sideset{}{_{k=1}^K}\sum \lambda_k W_{\epsilon}(\bm{q},\bm{p}_k).
\end{aligned}
\end{eqnarray}
Comparing \eqref{eq:bw_iter} and \eqref{eq:sink-wb}, the minimization in each proximal point iteration in \eqref{eq:bw_iter} can be solved by Bregman iterative projection~\citep{benamou2015iterative} using the same change-of-variable technique in Section \ref{section3}. 

\begin{algorithm}[!t]
\label{alg:barycenter}
\caption{IPOT-WB($\{\bm{p}_k\}$)}
\begin{algorithmic}[1]\label{algo_proximal}
\STATE \textbf{Input:} The probability vector set $\{\bm{p}_k\}$ on grid $\{y_i\}_{i=1}^n$
\STATE $\bm{b}_k\leftarrow \frac{1}{n}\textbf{1}_n,\forall k=1,2,...,K$
\STATE $C_{ij}\leftarrow c(y_i,y_j):=||y_i-y_j||^2_2$
\STATE $G_{ij}\leftarrow e^{-\frac{C_{ij}}{\beta}}$
\STATE $\bm{\Gamma}_k \leftarrow  \mathbf{11}^T$
\FOR {$t=1,2,3,...$}
\STATE $\bm{H}_k \leftarrow  \bm{G} \odot \bm{\Gamma}_k,\forall k=1,2,...,K$
\FOR {$l=1,2,3,...,L$}
\STATE $\bm{a}_k\leftarrow \frac{\bm{q}}{\bm{H}_k \bm{b}_k},\forall k=1,2,...,K$, 
\STATE $\bm{b}_k\leftarrow \frac{\bm{p}_k}{\bm{H}_k^T \bm{a}_k},\forall k=1,2,...,K$
\STATE $\bm{q} \leftarrow  \prod_{k=1}^K (\bm{a}_k \odot (\bm{H}_k \bm{b}_k))^{\lambda_k}$
\ENDFOR
\STATE $\bm{\Gamma}_k \leftarrow  \text{diag}(\bm{a}_k) \bm{H}_k \text{diag}(\bm{b}_k),\forall k=1,2,...,K$
\ENDFOR
\STATE  \textbf{Return} $\bm{q}$
\end{algorithmic}
\end{algorithm}

We provide the detailed algorithm in Algorithm \ref{algo_proximal}, and name this algorithm \textit{IPOT-WB}. Same as Algorithms~\ref{algo1}, IPOT-WB algorithm can converge with $L=1$ and a large range of $\beta$. 

Since the sketch in Figure \ref{fig:sketch} does not have restrictions on $f$ and $\mathcal{X}$, the sketch and the corresponding analysis for IPOT also applies to IPOT-WB, except the distance is in sense of convex combination of Bregman divergences instead of a single Bregman divergence.  


\vspace{-5pt}
\section{THEORETICAL ANALYSIS}
\label{section32}

\vspace{-5pt}

Classical proximal point algorithm has sublinear convergence rate. However, after we replace the square of Euclidean distance in classical proximal point algorithm by Bregman distance, we can prove stronger convergence rate -- a linear rate for both IPOT and IPOT-WB. First, we consider when the optimization problem (\ref{proximal3}) is solved exactly, we have a linear convergence rate guaranteed by the following theorem.
\begin{theorem}
Let $\{x^{(t)}\}$ be a sequence generated by the proximal point algorithm 
\begin{eqnarray*}
\begin{aligned}
x^{(t+1)} = &\sideset{}{_{x\in \mathcal{X}}}\argmin f(x) + \beta^{(t)} D_h(x,x^{(t)}),
\end{aligned}
\end{eqnarray*}
where $f$ is continuous and convex. Assume $f^* = \min f(x) > -\infty$. Then, with $\sum_{t=0}^{\infty} \beta^{(t)}=\infty$, we have
\begin{equation*}
\label{mono-converge}
f(x^{(t)})\downarrow f^*.
\end{equation*}
If we further assume $f$ is linear and $\mathcal{X}$ is bounded, the algorithm has linear convergence rate.
\end{theorem}
More importantly, the following theorem gives us a guarantee of convergence when (\ref{proximal3}) is solved inexactly.
 
\begin{theorem} \label{thom2}
Let $\{ x^{(t)} \}$ be the sequence generated by the Bregman distance based proximal point algorithm with inexact scheme (i.e., finite number of inner iterations are employed).
Define an error sequence $\{ e^{(t)} \}$ where
\begin{eqnarray*}
\begin{aligned}
e^{(t+1)} \in  \beta^{(t)}  \left[ \nabla  f(x^{(t+1)} ) +  \partial \iota_{\mathcal{X}} (x^{(t+1)} ) \right] 
 +  \left[ \nabla h(x^{(t+1)} ) - \nabla h(x^{(t)} ) \right],
\end{aligned}
\end{eqnarray*}
where $\iota_{\mathcal{X}}$ is the indicator function of set $\mathcal{X}$, and $\partial_{l_X} (\cdot)$ is the subdifferential of the indicator function $l_{X}$. If the sequence $\{ e^{k} \}$ satisfies $\sum_{k=1}^{\infty} \| e^{k} \| < \infty$ and $\sum_{k=1}^{\infty} \langle e^{k}, x^{(t)} \rangle$ exists and is finite, then $\{ x^{(t)} \}$ converges to $x^{\infty}$ with $f(x^{\infty}) = f^*$.
If the sequence $\{ e^{(t)} \}$ satisfies that exist $\rho\in (0,1)$ such that$ \| e^{(t)} \| \le \rho^t$, $ \langle e^{(t)}, x^{(t)} \rangle \le \rho^t$ and with assumptions that $f$ is linear and $\mathcal{X}$ is bounded, then $\{ x^{(t)} \}$ converges linearly.
\end{theorem}

The proofs of both theorems are given in the supplementary material.
Theorem~\ref{thom2} guarantees the convergence of inexact proximal point method --- as long as the inner iteration number $L$ satisfies 
the given conditions, IPOT and IPOT-WB algorithm would converge linearly. Note that although Theorem~\ref{thom2} manages to prove the linear convergence in inexact case, in practice the conditions is not trivial to verify. In practice we usually just adopt $L=1$. 

Now we know IPOT and IPOT-WB can converge to the exact Wasserstein distance and Wasserstein barycenter. What if an entropic regularization is wanted? Please refer to the supplementary material for how IPOT can achieve regularizations with early stopping.

\begin{figure} 
  \centering 
  \subfigure[$\beta=1$ with different $L$]{ 
    \label{fig:subfig:b} 
    \includegraphics[width=0.38\linewidth]{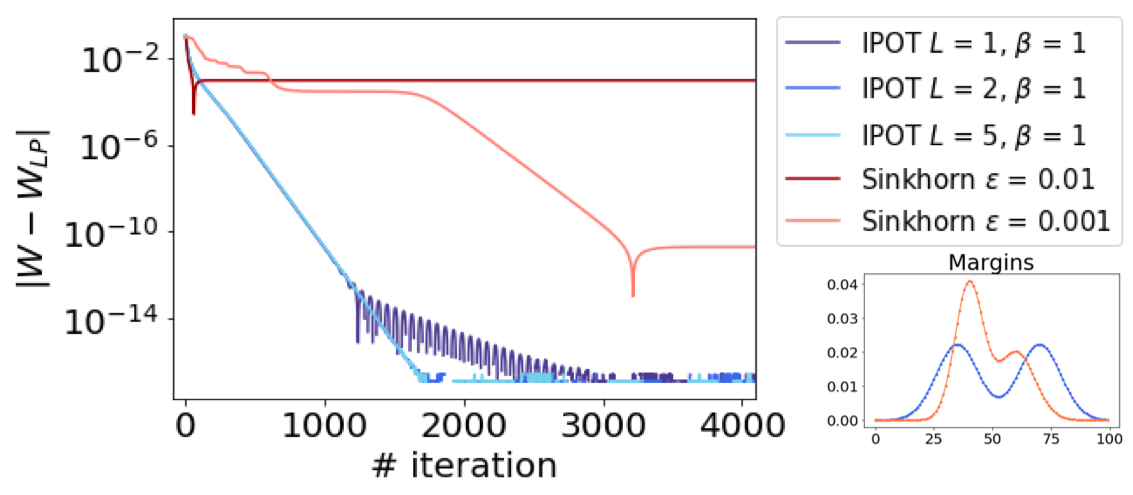}} 
  \subfigure[$L=1$ with different $\beta$]{ 
    \label{fig:subfig:b} 
    \includegraphics[width=0.38\linewidth]{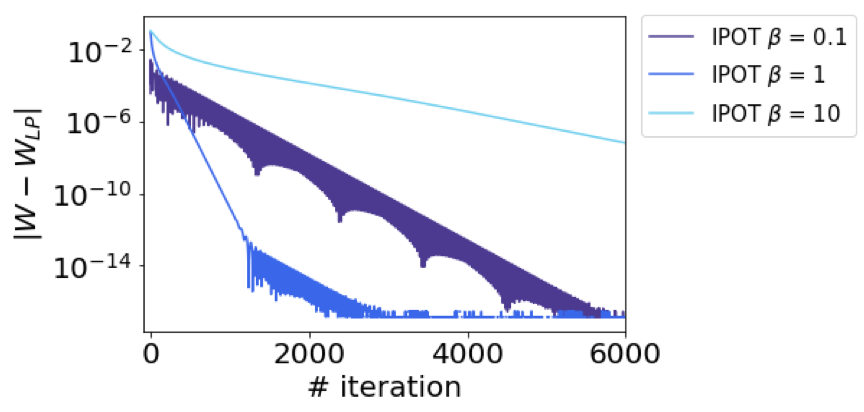}} 
  \caption{\label{fig:convergence} The plots of differences in computed Wasserstein distances w.r.t. number of iterations. Here, $W$ are the Wasserstein distance computed at current iteration. $W_{LP}$ is computed by simplex method, and is used as ground truth. The test adopts $c(x,y) = ||x-y||_2$. (a) The plot of the convergence trajectories of IPOT with different $L$. The right lower figure is the two input margins for the test. We also plot the ones for Sinkhorn method in comparison. (b) The plot of differences in computed Wasserstein distances with different $\beta$.} 
  \label{fig:subfig} 
\end{figure}


\vspace{-5pt}
\section{EMPIRICAL ANALYSIS}
\label{section4}

\vspace{-5pt}

In this section we will illustrate the convergence behavior with respect to inner iteration number $L$ and parameter $\beta$,  the scalability of IPOT, and the issue with entropy regularization.
We leverage the implementation of Sinkhorn iteration and LP solver based on Python package POT \citep{flamary2017pot}, and use Pytorch to parallel some of the implementation.




\vspace{-5pt}
\subsection{Convergence Rate}
\label{section41}
\vspace{-5pt}

A simple illustration task of calculating the Wasserstein distance of two 1D distribution is conducted as numerical validation of the convergence theorems proved in Section \ref{section32}. The two input margins are mixtures of Gaussian distributions shown in the figure in the right lower of Figure \ref{fig:convergence} (a): the red one is $0.4\phi(\cdot|60,8)+0.6\phi(\cdot|40,6)$, and the blue one is $0.5\phi(\cdot|35,9)+0.5\phi(\cdot|70,9)$, where $\phi(\cdot|\mu,\sigma^2)$ is the probability density function of 1 dimensional Gaussian distribution with mean $\mu$ and variance $\sigma^2$. Input vectors $\bm{\mu}$ and $\bm{\nu}$ is the two function values on the uniform discretization of interval $[1,100]$ with grid size 1.  To be clear, the use of two 1D distribution is only for visualization purpose. We also did tests on empirical distribution of 64D Gaussian distributed data, and the result shows the same trend. We include more discussion in the supplementary material.

Figure \ref{fig:convergence} shows the convergence of IPOT under different $L$ and $\beta$. We also include the result of Sinkhorn method for comparison. IPOT algorithm has empirically linear convergence rate even under very small $L$.

The convergence rate increases w.r.t. $\beta$ when $\beta$ is small, and decreases when $\beta$ is large. This is because the choice of $\beta$ is a trade-off between inner and outer convergence rates.  On the one hand, a smaller $\beta$ usually lead to quicker convergence of proximal point iterations. On the other hand, the convergence of inner Sinkhorn iteration, is quicker when $\beta$ is large. 

Furthermore, the choice of $L$ also appears to be a trade-off. While a larger $L$ takes more resources in each step, it also achieves a better accuracy, so less proximal point iterations are needed to converge. So the choice of best $L$ is relevant to the choice of $\beta$. For large $\beta$, the inner Sinkhorn iteration can converge faster, so smaller $L$ should be used. For small $\beta$, larger $L$ should be used, which is not efficient, and also improve the risk of underflow for the inner Sinkhorn algorithm. So unless there are specific need for accuracy, we do not recommend using very small $\beta$ and large $L$. 

For simplicity, we use $L=1$ for later tests. 

%

\begin{wrapfigure}{R}{0.5\textwidth}
\vspace{-20pt}
  \begin{center}
    \includegraphics[width=0.48\textwidth]{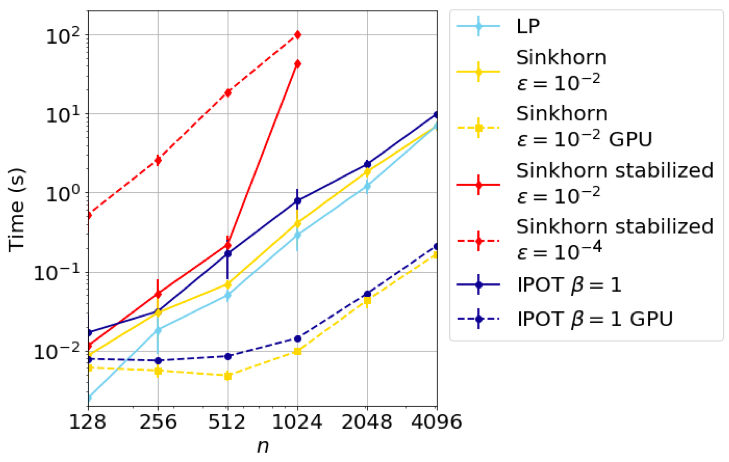}
    \end{center}
      \vspace{-20pt}
\caption{\label{fig:scalability} Log-log plot of average time used to achieve 1e-4 relative precision with error bar. Each point is obtained by the average of 6 tests on different datasets.}
  \vspace{-0pt}
\end{wrapfigure}

We conduct the following scalability test to show the computation time of the proposed IPOT comparing to the state-of-art benchmarks. The optimal transport problem is conducted between the two empirical distributions of 16D uniformly distributed data (See Section \ref{section21} for formulation). Besides proposed IPOT algorithm (see Algorithm \ref{algo1} with $L=1$), the Sinkhorn algorithm follows~\citet{cuturi2013sinkhorn} and the stabilized Sinkhorn algorithm follows~\citet{chizat2016scaling}. 
The result of the scalability test is shown in Figure \ref{fig:scalability}.
The LP solver has a good performance under the current experiment settings. But LP solver is not guaranteed to have good scalability as shown here. Moreover, LP method is difficult to parallel. Readers who are interested please refer to experiments in~\citet{cuturi2013sinkhorn}.

Sinkhorn and IPOT can be paralleled conveniently, so we provide both CPU and GPU tests here. Under this setting, IPOT takes approximately the same resources as Sinkhorn at $\epsilon=0.01$. For smaller $\epsilon$, original Sinkhorn will underflow, and we need to use stabilized Sinkhorn. Stabilized Sinkhorn is much more expensive than IPOT, especially for large datasets and small $\epsilon$, as demonstrated by the experiment result of stabilized Sinkhorn at $\epsilon=10^{-2}$ and $10^{-4}$. 

Note that we also try to use the method proposed in~\citet{schmitzer2016stabilized} for $\epsilon$ scaling, to help the convergence when $\epsilon \to 0$. However, although it is faster than Sinkhorn method when data size is smaller than $1024$, the time used at 1024 is already around $2\times 10^3$s. Therefore we didn't include this method in the figure.


\subsection{Effect of Entropy Regularization}

We have shown that IPOT can converge to exact Wasserstein distance with complexity comparable to Sinkhorn (see Figure \ref{fig:sketch} and \ref{fig:scalability}) and as we claimed in Section \ref{section1} this is important in some of the learning problems.

But in what cases is the exact Wasserstein distance truly needed? How will the entropy regularization term affect the result in different applications? In this section, we will discuss the exact transport plan with sparsity preference and the advantage of exact Wasserstein distance in learning the generative models.

\subsubsection{Sparsity of the Transport Plan}
\label{section43}

\begin{wrapfigure}{R}{0.54\textwidth}
\vspace{-20pt}
  \begin{center}
    \includegraphics[width=0.5\textwidth]{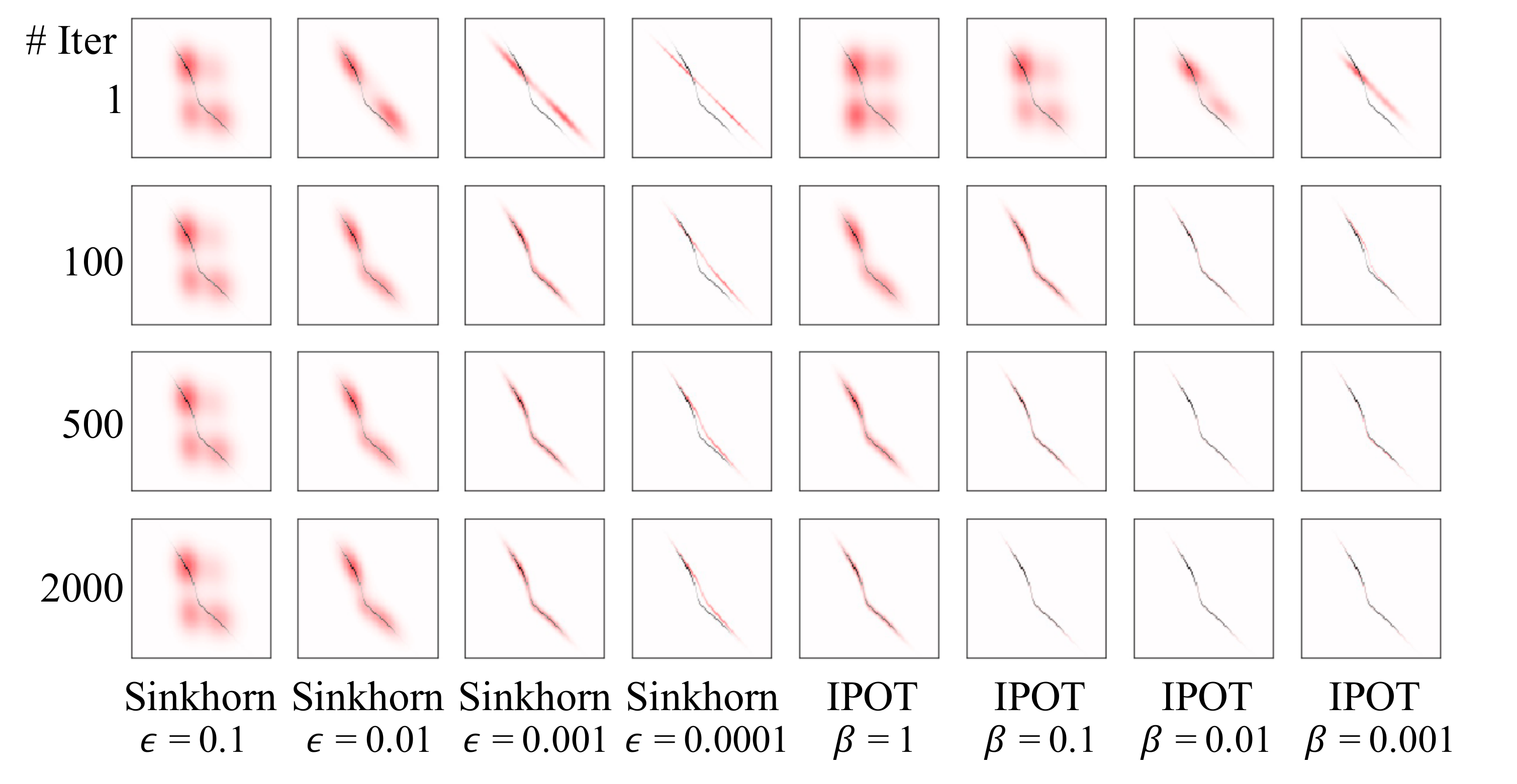} \end{center}
    \vspace{-20pt}
  \caption{\label{fig:joint_dist} The transport plan generated by Sinkhorn and IPOT methods at different iteration number. The red colormap is the result from Sinkhorn or IPOT method, while the black wire is the result of simplex method for comparison. In the right lower plans, the red and the black is almost identical.}
  \vspace{-0pt}
\end{wrapfigure}

%

In applications such as histogram calibration and color transferring, an  exact and sparse transport plan is wanted. In this section we conduct tests on the sparsity of the transport plan using the two distributions shown in Figure \ref{fig:convergence} for both IPOT and Sinkhorn methods with different regularization coefficients.
Figure \ref{fig:joint_dist} visualize the different transport plans. The red colormap is the result from Sinkhorn or IPOT method, where the black wire beneath is the result by simplex method as ground truth.  
To be clear, the different number of interaction of IPOT means the number of the outer iteration with still $L=1$ inner iteration.

 The proposed IPOT method can always converge to the sparse ground truth with enough iteration and it is very robust with respect to the parameter $\beta$, i.e., there is little visual difference with $\beta$ changing from $0.1$ to $0.001$. Furthermore, even with large $\beta=1$, the optimal plan is still sparse and acceptable. In addition, if some smoothness is wanted, IPOT method would also be able to work with early stopping. The degree of smoothness can be easily adjusted by adjusting the number of iterations if needed.

On the other hand, the optimal plans obtained by Sinkhorn has two issues. If the $\epsilon$ is chosen to be large (i.e., $\epsilon=0.1$ or $0.01$), the optimal plan are blur i.e.,  neither exact nor sparse.
In downstream applications, the non-sparse structure of transport plan make it difficult to extract the transport map from source distribution to target distribution. 
However if the $\epsilon$ is chosen to be small (i.e., $\epsilon=0.0001$), it needs more iterations to converge. For example, the Sinkhorn $\epsilon=0.0001$ case still cannot converge after 2000 iterations. 
So in Sinkhorn applications, $\epsilon$ needs to be selected carefully.
This fine tuning issue can be avoid by the proposed IPOT method, since IPOT is robust to the parameter $\beta$.


\begin{figure}[!t]
\centering
\includegraphics[width=0.95\textwidth]{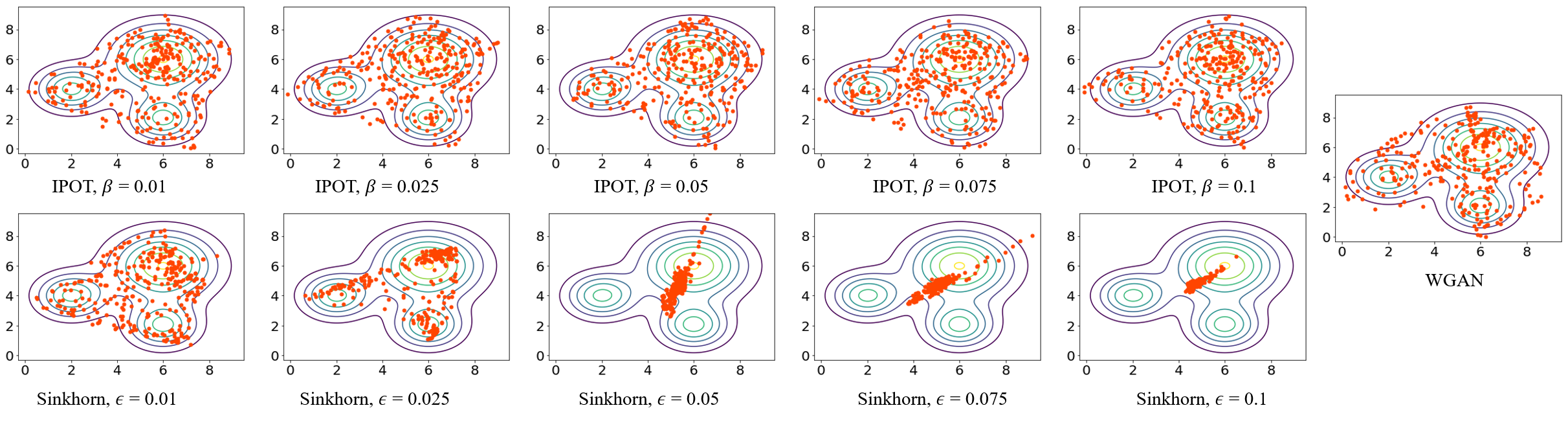}
\vspace{-0.2in}
\caption{\label{fig:generator2D} The sequences of learning results using IPOT, Sinkhorn, and original WGAN. In each figure, the orange dots are samples of generated data, while the contour represents the ground truth distribution. 
}
\vspace{-0.1in}
\end{figure}

\begin{figure}[!t]
\small
\centering
\subfigure[Sinkhorn $\epsilon=1$: digits 0,1,3,7,8,9 are covered]{
\includegraphics[width=0.4\linewidth]{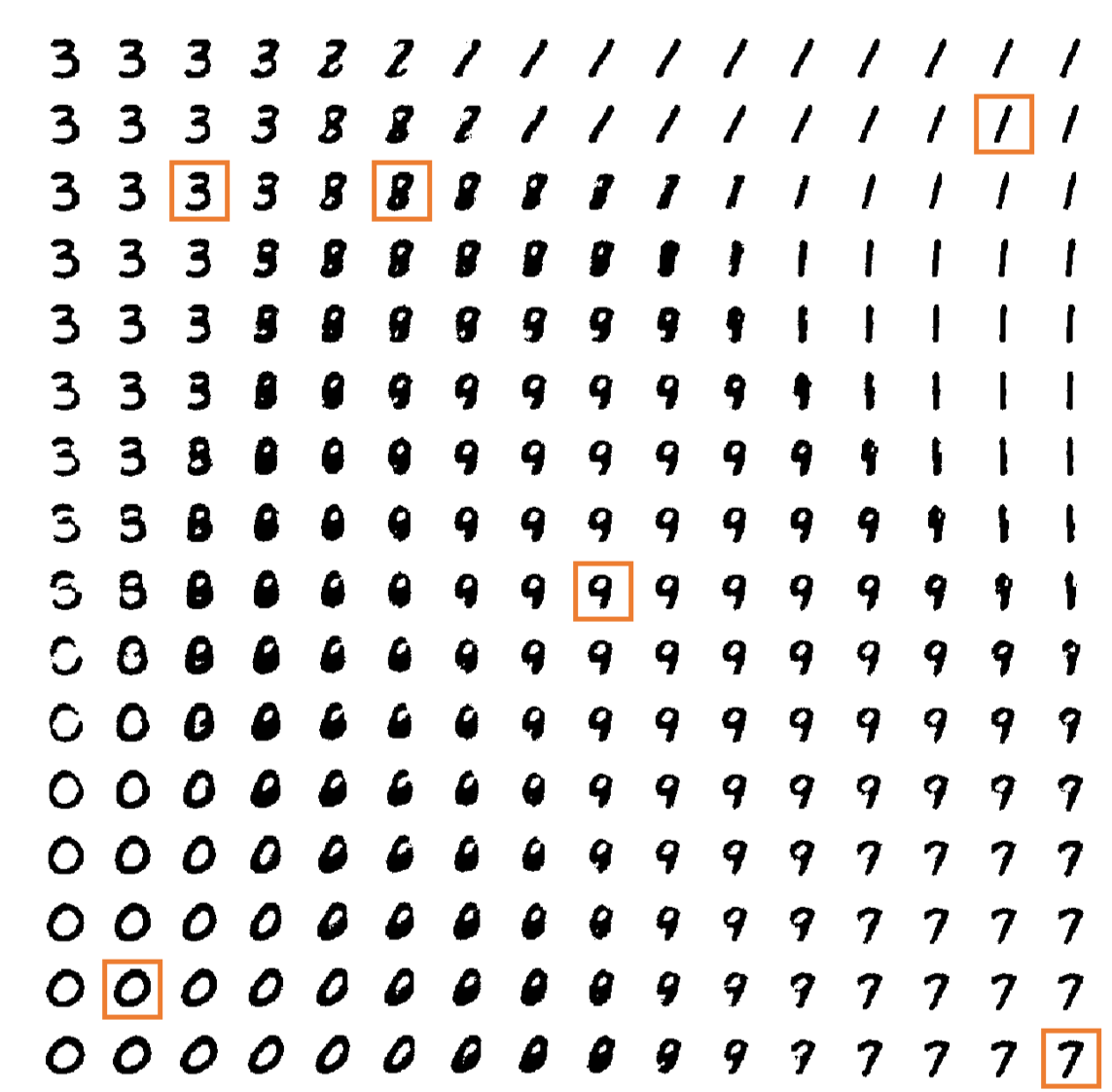}
} 
\subfigure[IPOT $\beta=1$: all digits are covered]{
\includegraphics[width=0.4\linewidth]{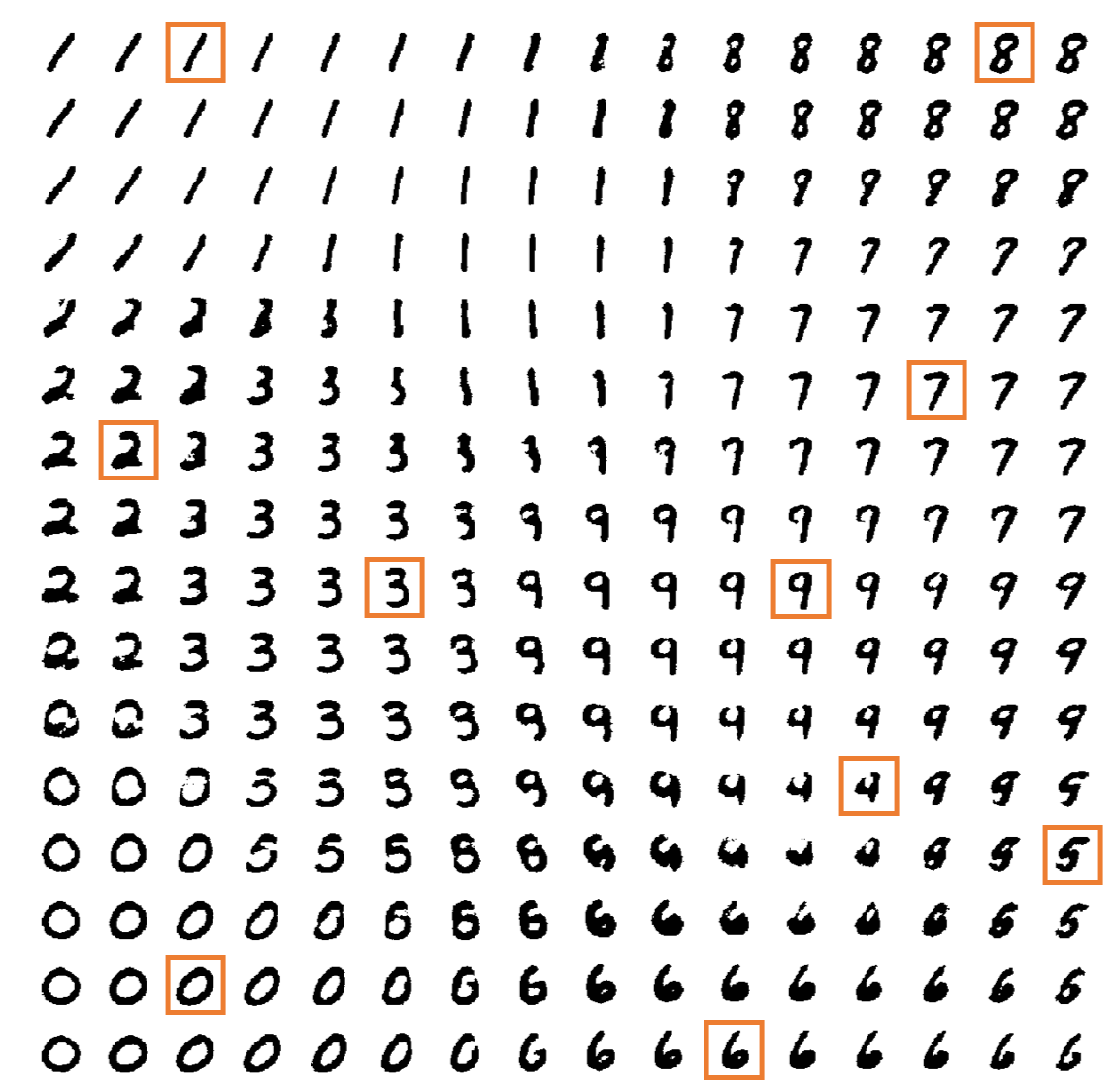}
}
\caption{\label{fig:generative_mnist} Plots of MNIST learning result under comparable resources. They both use batch size=200, number of hidden layer=1, number of nodes of hidden layer=500, number of iteration=200, learning rate = 1e-4. }
\end{figure}
\subsubsection{Shrinkage Problem in Generative Models}

\label{section432}
As shown in Equation (\ref{eq:sinkhorn}), Sinkhorn method use entropy to penalize the optimization target and has biased evaluation of Wasserstein distance. The inaccuracy will affect the performance of the learning problem where Wasserstein metric is served as loss function.

In order to better illustrate the affect of the inaccurate Wasserstein distance, we consider the task of learning generative models, specifically, Wasserstein GAN \citep{arjovsky2017wasserstein}. 
Similar to other GAN, WGAN seeks to learn a generated distrubution to approximate a target distribution, except using Wasserstein distance as the loss that measures the distance between the generated distribution and target distribution.
It uses the Kantorovitch dual formulation to compute Wasserstein distance.

In this section, we train a Wasserstein GAN with the dual formulation substituted by Sinkhorn and IPOT methods.
Detailed derivation can be found in supplementary material.  Meanwhile, the standard approach of using dual form proposed in \citet{arjovsky2017wasserstein} is also compared. Note that the purpose of this section is not to propose a new GAN but visualize how proposed IPOT can avoid the possible negative influence introduced by the inaccuracy of the entropy regularization in the Sinkhorn method.


We claim that result of Sinkhorn method with moderate size $\epsilon$ tends to shrink towards the mean, so the learned distribution cannot cover all the support of target distribution. To demonstrate the reason of this trend, consider the extreme condition when $\epsilon \to \infty$, the loss function becomes
\[
\bm{\Gamma^*} = \argmin_{\bm{\Gamma}} h(\bm{\Gamma}) = \argmin_{\bm{\Gamma}} D_h(\bm{\Gamma},\bm{11}^T/mn).
\]
So $\bm{\Gamma^*} = \bm{11}^T/mn$. If we view $\{x_i\}$ and $\{y_j:y_j=g_{\theta}(z_j)\}$ as the realizations of random variables $X$ and $Y$, the optimal Sinkhorn distance $W_{\epsilon}$ is expected to be
\begin{align*}
\mathbb{E}_{X,Y}[W_{\epsilon}] & = \mathbb{E}_{X,Y}[\langle \bm{\Gamma^*}, \bm{C} \rangle ]\\
& = \mathbb{E}_{X,Y}[ \sideset{}{_{i,j}} \sum ||x_i-y_j||_2^2] \\
& = n^2(\text{Var}(X) + (\overbar{X}-\overbar{Y})^2 +\text{Var}(Y)) ,
\end{align*}
where $n$ is the data size, $\overbar{(\cdot)}$ is the mean of random variable, and $\text{Var}(\cdot)$ is the variance. At the minimum of the distance, the mean of generated data $\{y_j\}$ is the same as $\{x_i\}$, but the variance is zero. Therefore, the learned distribution would shrink asymptotically toward the data mean due to smoothing the effect of regularization.

However, the proposed IPOT method is free from the above shrinking issue since the exact Wasserstein distance can be found with the approximately the same cost (see Section \ref{section41} and Section \ref{section42}). Now we illustrate the shrinkage problem by the following experiments.

\textit{\textbf{Experiments on 2D Synthetic Data} }
First, we conduct a 2D toy example to demonstrate the affect of regularization. We use a 2D-2D NN as generator to learn a mapping from uniformly distributed noise to mixture of Gaussian distributed real data. Since as shown in Figure \ref{fig:convergence}, Sinkhorn may need more iteration to converge, in this experiment, IPOT uses 200 iterations and Sinkhorn uses 500 iterations.

Figure \ref{fig:generator2D} shows the results. As $\epsilon$ varies from $0.01$ to $0.1$, the learned distribution of Sinkhorn gradually shrinks to the mean of target distribution, again this is because the inaccurcy in calculating the Wasserstein distance. On the contrary, since IPOT can converge to the exact Wasserstein distance regardless of different $\beta$, the result robustly cover the whole support of target distribution. Furthermore, comparing to the dual form method used in the WGAN, the proposed IPOT method is better in small scale cases and can achieve similar performance in large scale cases \citep{genevay2017sinkhorn}. This is mainly because the discriminator neural network used in WGAN is susceptible to overfitting in low dimensional cases, and it exceeds the objective of this paper.

\textit{\textbf{Experiments on Higher Dimensional Data} }
For higher dimensional data, we cannot visualize the final generated distribution as done in the 2D test. 
 So in order to demonstrate IPOT has little shrinkage issue, we set the latent space to be 2D, and visualize it by plotting the images generated at dense grid points on the latent space. Due to the low dimensional latent space, We perform the experiment using MNIST dataset. Note that this is mainly for the convenience in visualization, the whole shrinkage-free property of IPOT method is also extendable to more complex learning problems. Associated with the MNIST dataset, we use a generator
$g_{\theta}:\mathbb{R}^2 \mapsto \mathbb{R}^{784}$, noise data $\{z_j\}\sim \text{Unif}([0,1]^2)$ as input, and one fully connected hidden layer with 500 nodes. 

Figure \ref{fig:generative_mnist} shows an example of generated results. The Sinkhorn results look authentic, but we can only find some of the digits in it. This is exactly the consequence of shrinkage due to the inaccurate calculation of Wasserstein distance - in the domain where the density of learned distribution is nonzero, the density of target distribution is usually nonzero; but in some part of the domain where the density of target distribution is nonzero, the learned distribution is zero. In the example of Figure \ref{fig:generative_mnist} (a), the learned distribution cannot cover the support of digits 2,4,5,6 while when using IPOT to calculate the Wasserstein loss, all ten digits are can be recovered in \ref{fig:generative_mnist} (b), which shows the coverage of the whole domain of the target distribution. 
In supplementary material we provide more examples, e.g., if a larger $\epsilon$ is used, Sinkhorn generator would shrink to one point, and hence cannot learn anything, while the IPOT method is robust to its parameter $\beta$ and covers more digits.





\begin{figure}[t]
  \begin{center}
  \hspace*{-0.5cm}
    \includegraphics[width=0.7\textwidth]{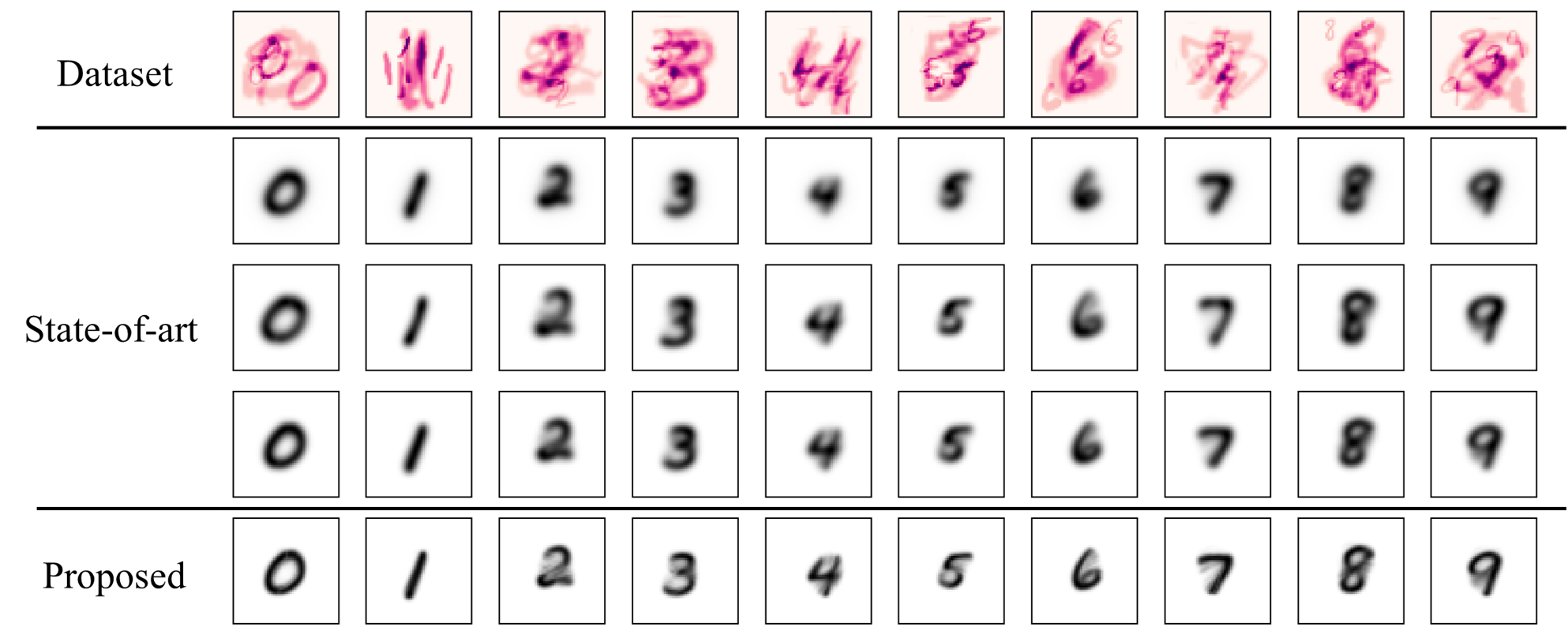}  
    \end{center}
    \vspace{-20pt}
  \caption{\label{fig:sample} The result of barycenter. For each digit, we randomly choose 8 of 50 scaled and shifted images to demonstrate the input data. From the top to the bottom, we show (top row) the demo of input data; (second row) the results based on~\citet{cuturi2014fast}; (third row) the result based on~\citet{solomon2015convolutional}; (fourth row) the result based on~\citet{benamou2015iterative}; (bottom row) the results based on IPOT-WB.
}
\vspace{-10pt}
\end{figure}

\subsection{Computing Barycenter}

We test our proximal point barycenter algorithm on MNIST dataset, borrowing the idea from \citet{cuturi2014fast}. Here, the images in MNIST dataset is randomly uniformly reshape to half to double of its original size, and the reshaped images have random bias towards corner. After that, the images are mapped into $50\times 50$ grid. For each digit we use 50 of the reshaped images with the same weights as the dataset to compute the barycenter. All results are computed using 50 iterations and under $\epsilon, \beta=0.001$. So for proximal point method, the regularization is approximately the same as $\epsilon=2\times 10^{-5}$, which is pretty small. We compare our method with state-of-art Sinkhorn based methods \citet{cuturi2014fast}, \citet{solomon2015convolutional} and \citet{benamou2015iterative}. Among the four methods, the convolutional method \citep{solomon2015convolutional} is different in terms of that it only handles structural input tested here and does not require $O(n^2)$ storage, unlike other three general purpose methods.

We are also aware of that there are other literatures for Wassersetin barycenter, such as \citet{staib2017parallel} and \citet{claici2018stochastic}, but they are targeting a more complicated setting, and has a different convergence rate (i.e. sublinear rate) than the methods we provide here.
The results (Figure \ref{fig:sample}) from proximal point algorithm are clear, while the results of Sinkhorn based algorithms suffer blurry effect due to entropic regularization. 
While the time complexity of our method is in the same order of magnitude with Sinkhorn algorithm~\citep{benamou2015iterative}, the space complexity is $K$ times of it, because $K$ different transport maps need to be stored. This might cause pressure to memory for large $K$. Therefore, a sequential method is needed. We left this to future work.


\section{CONCLUSION}

We proposed a proximal point method - IPOT - based on Bregman distance to solve optimal transport problem. Different from the Sinkhorn method, IPOT algorithm can converge to ground truth even if the inner optimization iteration only performs once. This nice property results in similar convergence and computation time comparing to Sinkhorn method. However, IPOT provides a robust and accurate computation of Wasserstein distance and associated transport plan, which leads to a better performance in image transformation and avoids the shrinkage in generative models.
We also apply the IPOT idea to calculate the Wasserstein barycenter. The proposed method can generate much sharper results than state-of-art due to the exact computation of the Wasserstein distance.

\section*{ACKNOWLEDGEMENT}

This work is partially supported by the grant NSF IIS 1717916, NSF CMMI 1745382, NSFC 61672231, NSFC U1609220 and 19ZR1414200.

\bibliographystyle{ims}
\bibliography{example_paper}
\newpage
\onecolumn

\par\noindent\rule{\textwidth}{1.4pt}

\vspace{5pt}

\begin{center}
    {\large{\textbf{A Fast Proximal Point Method for Computing Exact Wasserstein Distance -- Appendix}}}
\end{center}

\par\noindent\rule{\textwidth}{0.6pt}

\appendix
\section{More Analysis on IPOT}

\begin{wrapfigure}{r}{0.54\textwidth}
\vspace{-20pt}
  \begin{center}
    \includegraphics[width=0.5\textwidth]{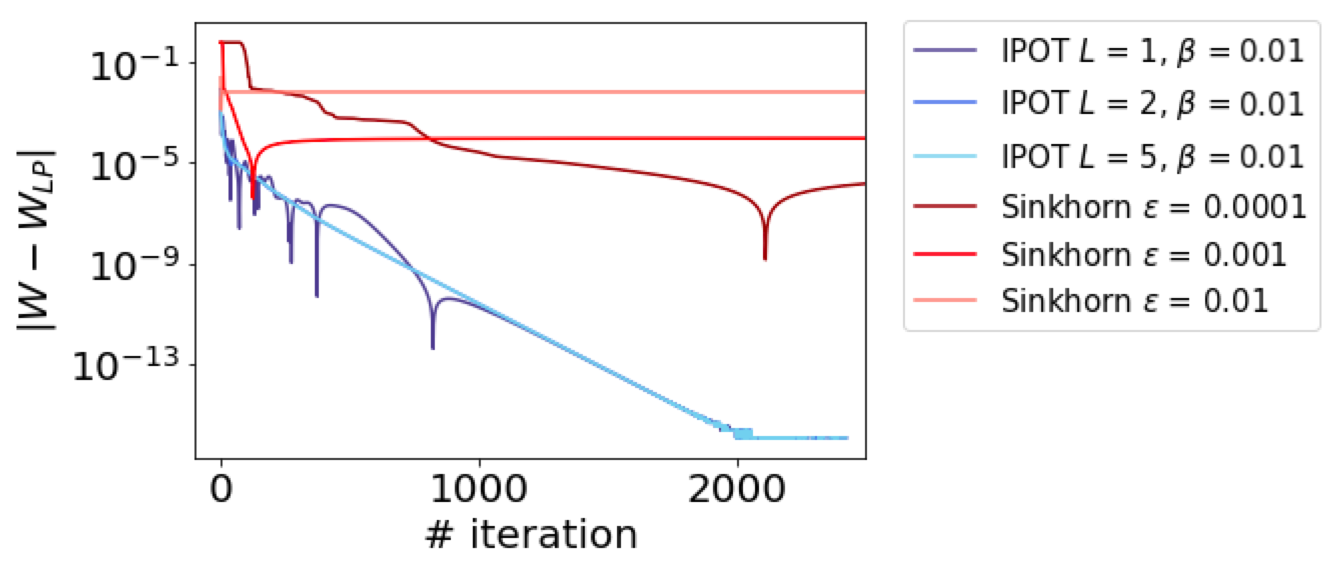}  \end{center}
      \vspace{-20pt}
  \caption{\label{fig:convergence64}The plot of differences in computed Wasserstein distances w.r.t. number of iterations for 64D Gaussian distributed data. Here, $W$ are the Wasserstein distance computed at current iteration. $W_{LP}$ is computed by simplex method, and is used as ground truth. The test adopts $c(x,y) = ||x-y||_2$. Due to random data is used, the number of iteration that the algorithm reaches $10^{-17}$ varies from 1000 to around 5000 according to our tests.}
  \vspace{-10pt}
\end{wrapfigure}

\subsection{Convergence w.r.t. $L$}

As mentioned in Section 6.1, we provide the test result of 64D Gaussian distributed data here. We choose the computed Wasserstein distance $\langle \bm{\Gamma}, \bm{C} \rangle$ as the indicator of convergence, because while the optimal transport plan might not be unique, the computed Wasserstein distance at convergence must be unique and minimized to ground truth. We use the empirical distribution as input distributions, i.e.,

\begin{eqnarray}
\begin{aligned}
&W(\{x_i\},\{g_{\theta}(z_j)\})= \sideset{}{_{\bm{\Gamma}}}\min \langle \bm{C}(\theta),\bm{\Gamma} \rangle \quad\\
&\text{s.t. } \bm{\Gamma} \pmb{1}_n = \frac{1}{n}\pmb{1}_n, \bm{\Gamma}^T \pmb{1}_n = \frac{1}{n} \pmb{1}_n.
\end{aligned}
\end{eqnarray}

As shown in Figure \ref{fig:convergence64}, the convergence rate is also linear. For comparison, we also provide the convergence path of Sinkhorn iteration. The result cannot converge to ground truth because the method is essentially regularized.

{\bf Remark.} When we are talking about amount of regularization, usually we are referring to the magnitude of $\epsilon$ for Sinkhorn, or the equivalent magnitude of $\epsilon$ computed from remark in Section 3 for IPOT method. However, the amount of regularization in a loss function should be quantified by $\epsilon/||\bm{C}||$, instead of $\epsilon$ alone. That is why in this paper, different magnitude of $\epsilon$ is used for different application.

\subsection{How IPOT Avoids Instability}

 Heuristically, if Sinkhorn does not underflow, with enough iteration, the result of IPOT is approximately the same as Sinkhorn with $\epsilon^{(t)}=\beta/t$. 
The difference lies in IPOT is a principled way to avoid underflow and can converge to arbitrarily small regularization, while Sinkhorn always causes numerical difficulty when $\epsilon\to 0$, even with scheduled decreasing $\epsilon$ like \citet{chizat2016scaling}. More specifically, in IPOT, we can factor $\Gamma=\text{diag}(\bm{u}_1)\bm{G}^{t}\text{diag}(\bm{u}_2)$, where $(\cdot)^t$ is element-wise exponent operation, and $\bm{u}_1$ and $\bm{u}_2$ are two scaling vectors. So we have $\epsilon^{(t)}=\beta/t$. As $t$ goes infinity, all entries of $\bm{G}^{t}$ would underflow if we use Sinkhorn with $\epsilon^{(t)}=\beta/t$. But we know $\bm{\Gamma}^*$ is neither all zeros nor contains infinity. So instead of computing $\bm{G}^{t}$, $\bm{u}_1$ and $\bm{u}_2$ directly, we use $\bm{\Gamma}^{t}$ to record the multiplication of $\bm{G}^t$ with part of $\bm{u}_1$ and $\bm{u}_2$ in each step, so the entries of $\bm{\Gamma}^{t}$ will not over/underflow. The explicit computation of $\bm{G}^{t}$ is not needed.

Therefore,  by tuning $\beta$ and iteration number, we can achieve the result of arbitrary amount of regularization with IPOT.

\section{Learning Generative Models}
In this section, we show the derivation for the learning algorithm, and more tests result.

For simplicity, we assume $|\{x_i\}|=|\{z_j\}|=n$. 
Given a dataset $\{x_i\}$ and some noise $\{z_j\}$ \citep{bassetti2006minimum, goodfellow2014generative}, our goal is to find a parameterized function $g_{\theta}(\cdot)$ that minimize $W(\{x_i\}, \{g_{\theta}(z_j)\})$, 
\begin{eqnarray}\label{lossG}
\begin{aligned}
W(\{x_i\},\{g_{\theta}(z_j)\}) & = \sideset{}{_{\bm{\Gamma}}}\min \langle \bm{C}(\theta),\bm{\Gamma} \rangle \quad \\
& \text{s.t. } \bm{\Gamma} \pmb{1}_n = \frac{1}{n}\pmb{1}_n, \bm{\Gamma}^T \pmb{1}_n = \frac{1}{n} \pmb{1}_n,
\end{aligned}
\end{eqnarray}
where $\bm{C}(\theta)=[c(x_i,g_{\theta}(z_j))]$.
Usually, $g_{\theta}$ is parameterized by a neural network with parameter $\theta$, and the minimization over $\theta$ is done by stochastic gradient descent.

In particular, given current estimation $\theta$, we can obtain optimum $\bm{\Gamma}^*$ by IPOT, and compute the Wasserstein distance by $\langle \bm{C}(\theta),\bm{\Gamma}^* \rangle$ accordingly. 
Then, we can further update $\theta$ by the gradient of current Wasserstein distance. There are two ways to solve the gradient: One is auto-diff based method such as \citet{genevay2017sinkhorn}, the other is based on the envelope theorem~\citep{afriat1971theory}. Different from the auto-diff based methods, the back-propagation based on envelope theorem does not go into proximal point iterations because the derivative over $\bm{\Gamma}^*$ is not needed, which accelerates the learning process greatly. This also has significant implications numerically because the derivative of a computed quantity tends to amplify the error. Therefore, we adopt envelope based method.

\begin{theorem}
\textbf{Envelope theorem.} Let $ f(x,\theta )$ and $l(x)$ be real-valued continuously differentiable functions, where $x\in \mathbb {R} ^{n}$ are choice variables and $ \theta \in \mathbb {R} ^{m}$ are parameters. Denote $x^*$ to be the optimal solution of $f$ with constraint $l=0$ and fixed $\theta$, i.e.
\[ x^* = \argmin _{x}f(x,\theta ) \quad s.t. \quad  l(x)= 0.\]
Then, assume that $V$ is continuously differentiable function defined as $V(\theta )\equiv f(x^{\ast }(\theta ),\theta)$, the derivative of $V$ over parameters is 
\[ \frac {\partial V(\theta)}{\partial \theta}=\frac{\partial f}{\partial \theta}.\]
\end{theorem}

In our case, because $\bm{\Gamma}^*$ is the minimization of $\langle \bm{\Gamma},\bm{C}(\theta)\rangle$ with constraints, we have
\begin{align*}
& \frac{\partial W(\{x_i\},\{g_{\theta}(z_j)\})}{\partial \theta}  = \frac {\partial \langle \bm{\Gamma}^*,\bm{C}(\theta)\rangle}{\partial \theta} \\
& = \langle \bm{\Gamma}^*,{\frac {\partial \bm{C}(\theta)}{\partial \theta}}\rangle=\langle \bm{\Gamma}^*,2(g_{\theta}(z_j)-x_i){\frac {\partial g_{\theta}(z_i)}{\partial \theta}}\rangle,
\end{align*}

where we assume $C_{ij}(\theta)=\|x_i-g_{\theta}(z_j)\|_2^2$, but the algorithm can also adopt other metrics. The derivation is in supplementary materials. The flowchart is shown in Figure \ref{fig:arche_more}, and the algorithm is shown in Algorithm~\ref{algo_gan}.

Note Sinkhorn distance is defined as $S(\{x_i\},\{g_{\theta}(z_j)\})= \langle \bm{C}(\theta),\bm{\Gamma}^* \rangle$,
where $\bm{\Gamma}^*=\sideset{}{_{\bm{\Gamma}\in \Sigma( \pmb{1}/n,\pmb{1}/n)}}\argmin \langle \bm{C}(\theta),\bm{\Gamma} \rangle+\epsilon h(\bm{\Gamma})$. If Sinkhorn distance is used in learning generative models, envelope theorem cannot be used because the loss function for optimizing $\theta$ and $\bm{\Gamma}$ is not the same. 

In the tests, we observe the method in~\citet{genevay2017sinkhorn} suffers from shrinkage problem, i.e. the generated distribution tends to shrink towards the target mean. The recovery of target distribution is sensitive to the weight of regularization term $\epsilon$. Only relatively small $\epsilon$ can lead to a reasonable generated distribution.

      \begin{algorithm}[H]
       \caption{Learning generative networks}
	\label{algo_gan}       
	 \begin{algorithmic}
          \STATE \textbf{Input:} real data $\{x_i\}$, initialized generator $g_\theta$
	\WHILE{not converged}
\STATE Sample a batch of real data $\{x_i\}_{i=1}^n$
\STATE Sample a batch of noise data $\{z_j\}_{i=1}^n \sim q$
\STATE $C_{ij}:=c(x_i,g_{\theta}(z_j)):=||x_i-g_{\theta}(z_j)||^2_2$
\STATE $\bm{\Gamma}$ = IPOT($\frac{1}{n}\bm{1}_n,\frac{1}{n}\bm{1}_n,\bm{C}$)
\STATE Update $\theta$ with $\langle \bm{\Gamma},[2(x_i-g_{\theta}(z_j)){\frac {\partial g_{\theta}(z_j)}{\partial \theta}}]\rangle$
\ENDWHILE       
 \end{algorithmic}
      \end{algorithm}

\begin{wrapfigure}{R}{0.4\textwidth}
\vspace{-20pt}
  \begin{center}
    \includegraphics[width=0.35\textwidth]{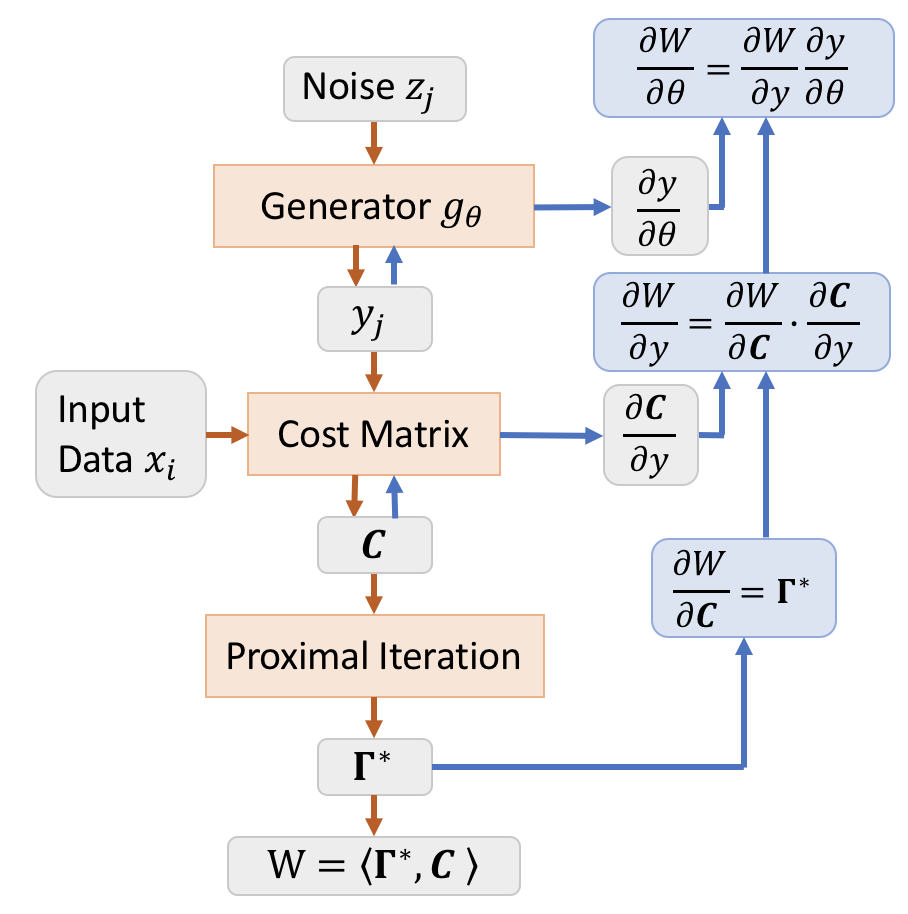} \end{center}
  \caption{\label{fig:arche_more} The architecture of the learning model using Envelope theorem in detail. According to Envelope theorem, we do not need to compute $\frac{\partial W}{\partial \bm{\Gamma}^*}$, so we do not need to back-propagate into the iteration.
}
  \vspace{-10pt}
\end{wrapfigure}

\subsection{Synthetic Test}

In section 5.1, we show the learning result of Sinkhorn and IPOT in 2D case. In Figure \ref{fig:generator1D_comp} we show sequences of results for a 1D-1D generator, respectively. The upper sequence is IPOT with $\beta=0.01,0.025,0.05,0.075,0.1$. The results barely change w.r.t. $\beta$. The lower sequence is the corresponding Sinkhorn results. The results shrink to the mean of target data, as expected. Also, we observe the learned distribution tends to have a tail that is not in the range of target data (also in 2D result, we do not include that part for a better view). It might be because the range of support that has a small probability has very small gradient when updated. Once the distribution is initialized to have a tail with small probability, it can hardly be updated. But this theory cannot explain why larger $\epsilon$ corresponds to longer tails. The tails can be on the left or right. We pick the ones on the left for easier comparison.

\begin{figure*}
\centering
\includegraphics[width=1.\textwidth]{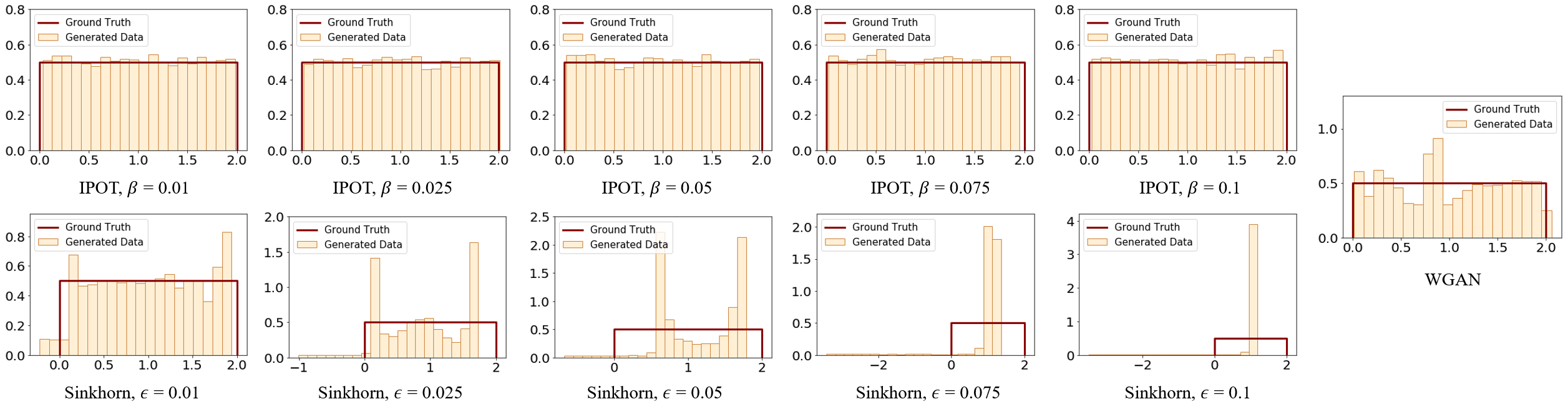}
\caption{\label{fig:generator1D_comp} The sequences of learning result of IPOT, Sinkhorn. In each figure, the orange histogram is the histogram of generated data, while the red line represents the PDF of the ground truth of target distribution. 
}
\end{figure*}

\begin{figure*}
\centering

\subfigure[IPOT $\beta=0.1$]{
\includegraphics[width=0.31\linewidth]{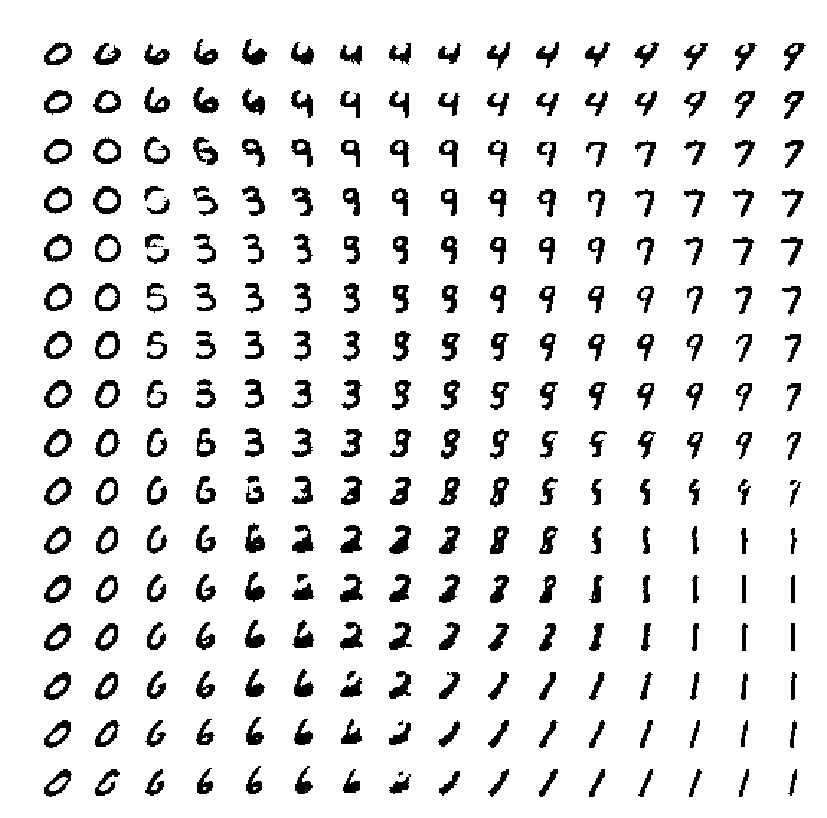}
} 
\subfigure[IPOT $\beta=1$]{
\includegraphics[width=0.31\linewidth]{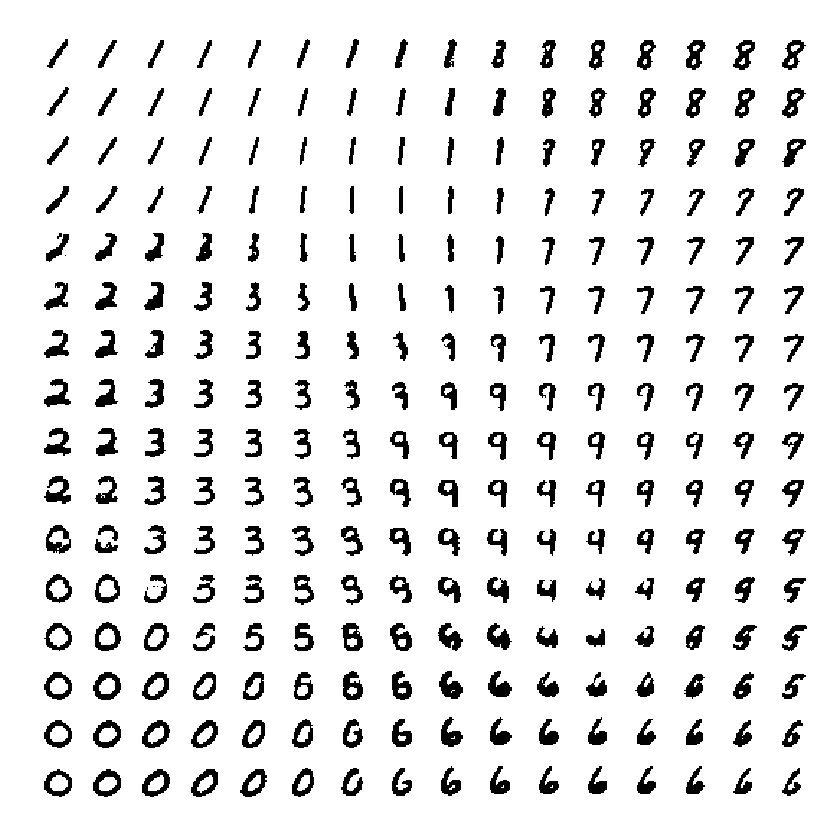}
} 
\subfigure[IPOT $\beta=10$]{
\includegraphics[width=0.31\linewidth]{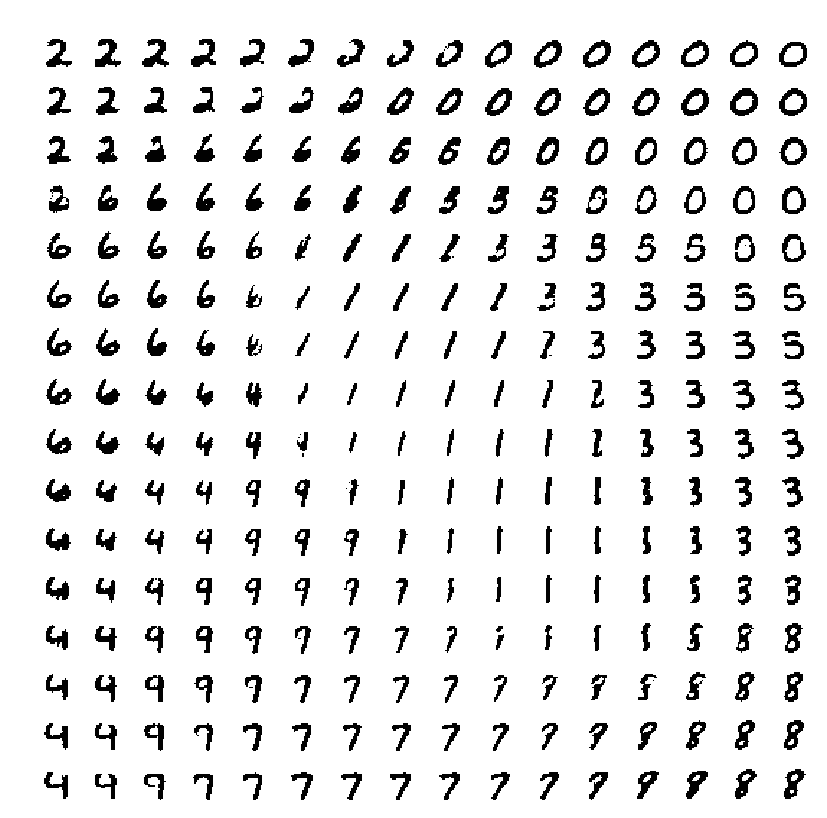}
} 
\subfigure[Sinkhorn $\epsilon=0.1$]{
\includegraphics[width=0.31\linewidth]{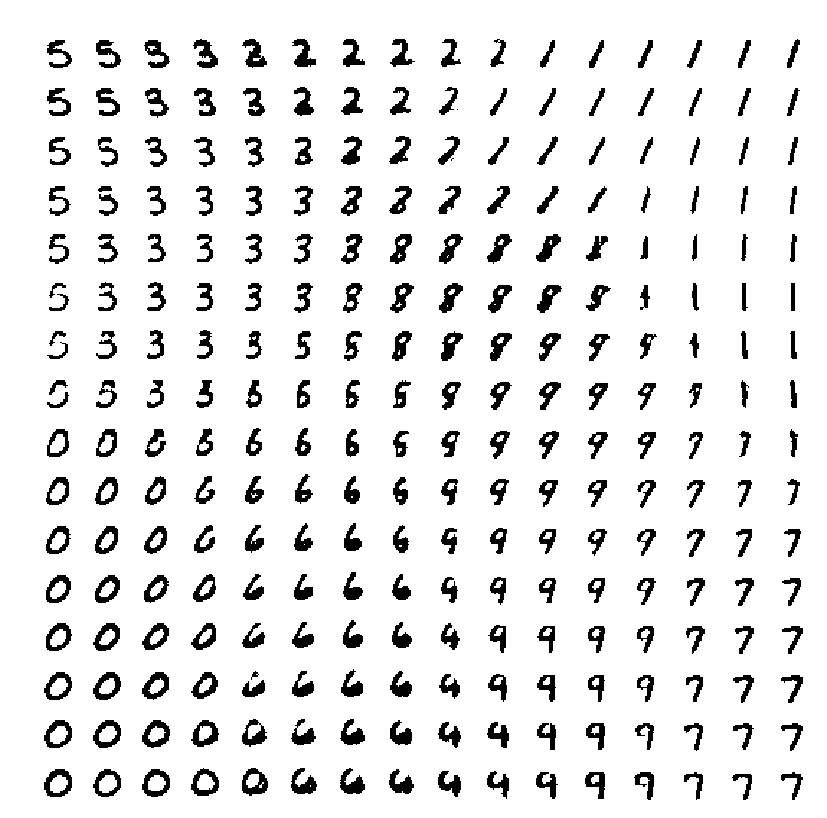}
} 
\subfigure[Sinkhorn $\epsilon=1$]{
\includegraphics[width=0.31\linewidth]{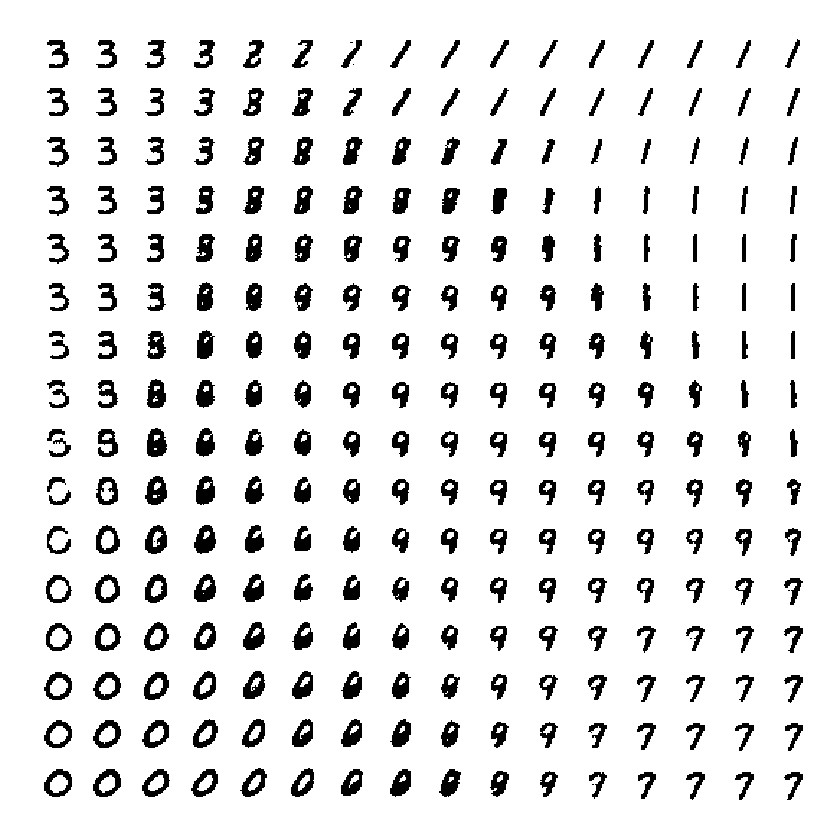}
} 
\subfigure[Sinkhorn $\epsilon=10$]{
\includegraphics[width=0.31\linewidth]{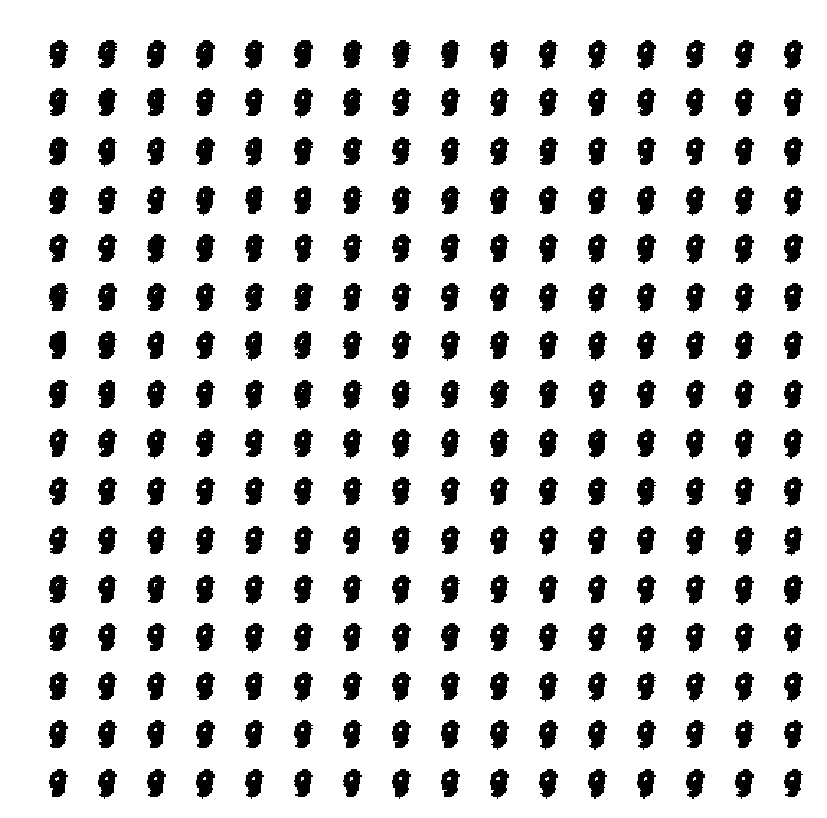}
} 
\caption{\label{fig:generative_mnist_comp} Plots of MNIST learning result under comparable resources with different $\epsilon$. They both use batch size=200, number of hidden layer=1, number of nodes of hidden layer=500, number of iteration=500, learning rate = 1e-4. Note that despite we show result of $\epsilon=0.1$ here, the algorithm does not run stably. It would sometimes fail due to numerical issue. }
\end{figure*}

\subsection{MNIST Test}

The same shrinkage can be observed in MNIST data as well. See figure \ref{fig:generative_mnist_comp}. While $\epsilon=0.1$ covers most shapes of the numbers, $\epsilon=1$ only covers a fraction, and $\epsilon=10$ seems to cover only the mean of images.

\begin{figure}
\centering
\subfigure[Original image]{
\includegraphics[width=0.25\linewidth]{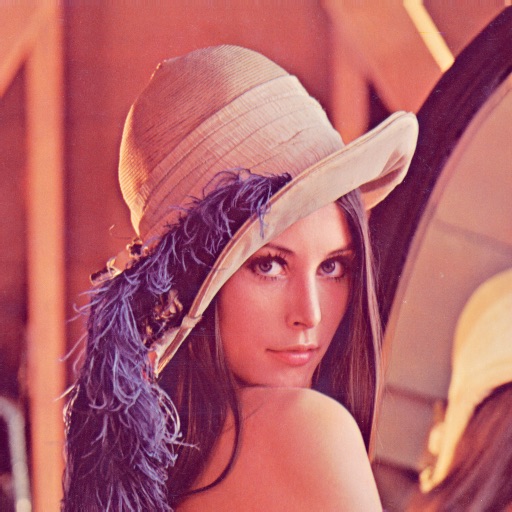}
}
\subfigure[Color image]{
\includegraphics[width=0.25\linewidth]{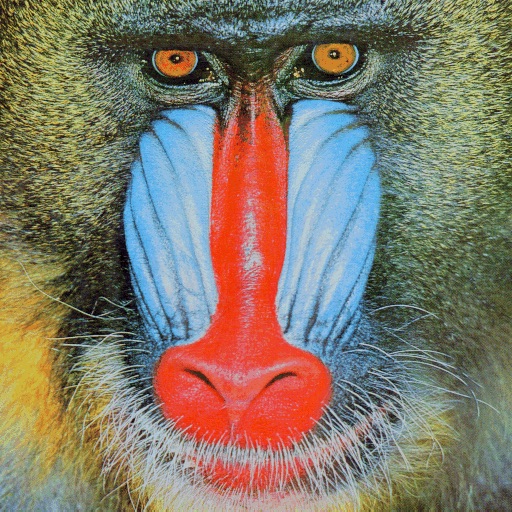}
}
\subfigure[Simplex]{
\includegraphics[width=0.25\linewidth]{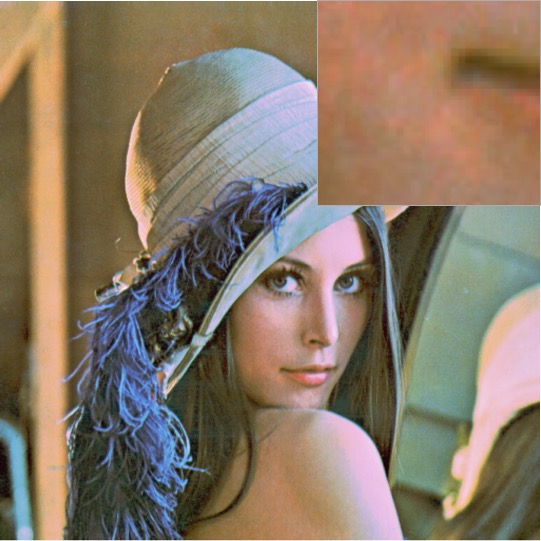}
}
\subfigure[Sinkhorn $\epsilon$=$10^{-2}$]{
\includegraphics[width=0.25\linewidth]{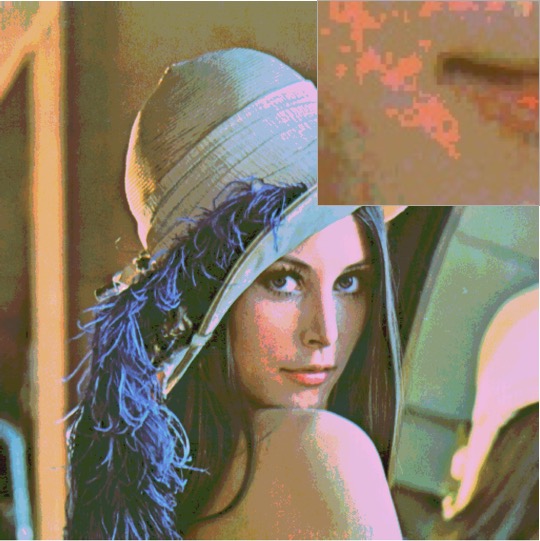}
}
\subfigure[Sinkhorn $\epsilon$=$10^{-3}$]{
\includegraphics[width=0.25\linewidth]{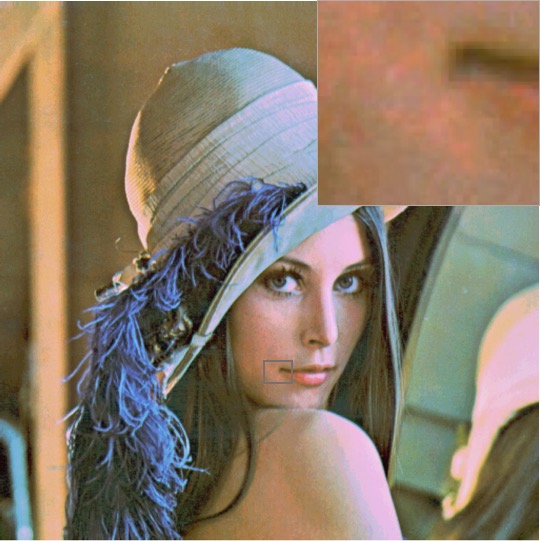}
} 
\subfigure[IPOT]{
\includegraphics[width=0.25\linewidth]{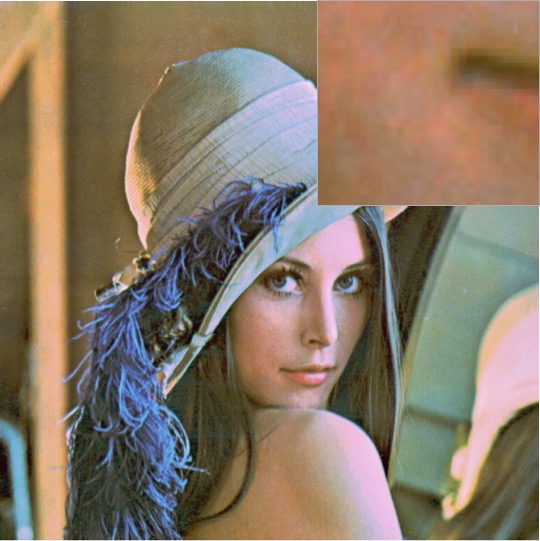}
}
\caption{\label{fig:color_transfer} An example of color transferring. The right upper corner of each generated image shows the zoom-in of the color detail of the mouth corner.}
\end{figure}

\section{Color Transferring}

Optimal transport is directly applicable to many applications, such as color transferring and histogram calibration. We will show the result of color transferring and why accurate transport plan is superior to entropically regularized ones.

The goal of color transferring is to transfer the tonality of a target image into a source image. This is usually done by imposing the histogram of the color palette of one image to another image. Since \citet{reinhard2001color}, many methods~\citep{rabin2014adaptive,xiao2006color} are developed to do so by learning the transformation between the two histograms. Experiments in~\citet{rubner1998metric} have shown that transformation based on optimal transport map outperforms state-of-the-art techniques for challenging images.

Same as other prime-form Wasserstein distance solvers~\citep{pele2009fast,cuturi2013sinkhorn}, the proximal point method provide a transport plan. By definition, the plan is a transport from the source distribution to a target one with minimum cost. Therefore it provides a way to transform a histogram to another.

One example is shown in figure \ref{fig:color_transfer}. We use three different maps to transform the RGB channels, respectively. For each channel, there are at most 256 bins. Therefore, using three channels separately is more efficient than treating the colors as 3D data. Figure \ref{fig:color_transfer} shows proximal point method can produce identical result as linear programming at convergence, while the results produced by Sinkhorn method differ w.r.t. $\epsilon$.

\section{General Bregman Proximal Point Algorithm}
In the main body of the paper, we discussed the proximal point algorithm with specific Bregman distance, which is generated through the traditional entropy function.
In this section, we generalize our results by proving the effectiveness of proximal point algorithm with general Bregman distance.
Bregman distance is applied to measure the discrepancy between different matrices which turns out to be one of the key ideas in regularized optimal transport problems.
Its special structure also give rise to proximal-type algorithms and projectors in solving optimization problems.

\subsection{Basic Algorithm Framework and Preliminaries}

The fundamental iterative scheme of general Bregman proximal point algorithm can be denoted as
\begin{equation}\label{Bregman-iteration}
x^{(t+1)} = \arg\min_{x\in X}\left\{ f(x) + \beta^{(t)} D_h (x , x^{(t)}) \right\},
\end{equation}
where $t\in \mathbb{N}$ is the index of iteration, and $D_h(x,x^{(t)})$ denotes a general Bregman distance between $x$ and $x^{(t)}$ based on a Legendre function $h$ (The definition is presented in the following).
In the main body of the paper, $h$ is specialized as the classical entropy function and as follows the related Bregman distance reduces to the generalized KL divergence.
Furthermore, the Sinkhorn-Knopp projection can be introduced to compute each iterative subproblem.
In the following, we present some fundamental definitions and lemmas.
\begin{definition}
Legendre function: Let $h: X\rightarrow (-\infty, \infty]$ be a lsc proper convex function. It is called
\begin{enumerate}
\item {\em{Essentially smooth}}: if $h$ is differentiable on $\hbox{int}\,\hbox{dom}\, h$, with moreover $\| \nabla h(x^{(t)}) \|\rightarrow \infty$ for every sequence $\{ x^{(t)} \}\subset \hbox{int}\,\hbox{dom}\, h$ converging to a boundary point of $\hbox{dom}\, h$ as $t\rightarrow +\infty$;
\item {\em{Legendre type}}: if $h$ is essentially smooth and strictly convex on $\hbox{int}\,\hbox{dom}\, h$.
\end{enumerate}
\end{definition}
\begin{definition}
Bregman distance: any given Legendre function $h$,
	\begin{equation}\label{Bregman-distance}
	D_h (x,y) = h (x) - h (y) - \langle \nabla h(y), x - y \rangle,\quad \forall x\in \hbox{dom}\,h, \forall y\in \hbox{int}\,\hbox{dom}\, h,
	\end{equation}where $D_h$ is strictly convex with respect to its first argument. Moreover, $D_h (x,y)\ge 0$ for all $(x,y)\in \hbox{dom}\, h\times \hbox{int}\, \hbox{dom}\, h$, and it is equal to zero if and only if $x = y$.
	However, $D_h$ is in general asymmetric, i.e., $D_h(x,y)\ne D_h(y,x)$.
\end{definition}
\begin{definition}
Symmetry Coefficient: Given a Legendre function $h: X\rightarrow (-\infty,\infty]$, its symmetry coefficient is defined by
\begin{equation}\label{Symmetry-coefficient}
\alpha (h) = \inf\left\{ \frac{D_h (x,y)}{D_h (y,x)}\ \big|\ (x,y)\in \hbox{int}\, \hbox{dom}\, h\times \hbox{int}\, \hbox{dom}\, h,\ x\ne y \right\}\in [0,1].
\end{equation}
\end{definition}
\begin{lemma}
Given $h:X\rightarrow (-\infty,+\infty]$, $D_h$ is general Bregman distance, and $x$, $y$, $z\in X$ such that $h(x)$, $h(y)$, $h(z)$ are finite and $h$ is differentiable at $y$ and $z$,
\begin{equation}\label{three_point_lemma}
D_h (x,z) - D_h (x,y) - D_h (y,z) = \left\langle \nabla h(y) - \nabla h (z), x - y \right\rangle
\end{equation}
\end{lemma}
\begin{proof}
The proof is straightforward as one can easily verify it by simply subtracting $D_h(y,z)$ and $D_h(x,y)$ from $D_h(x,z)$.
\end{proof}

\subsection{Theorem 5.1 and Theorem 5.2}
In this section, we first establish the convergence of Bregman proximal point algorithm, {\em{i.e.}}, {\bf{Theorem 5.1}}, while our analysis is based on \citet{eckstein1993nonlinear,eckstein1998approximate,teboulle1992entropic}.
Further, we establish the convergence of inexact version Bregman proximal point algorithm, {\em{i.e.}}, {\bf{Theorem 5.2}}, in which the subproblem in each iteration is computed inexactly within finite number of sub-iterations.

Note that here for simplicity we provide proof of $d(\bm{\Gamma},\bm{\Gamma}^{(t)})=D_h(\bm{\Gamma},\bm{\Gamma}^{(t)})$, i.e., the IPOT case. We can analogously prove it for $d(\{\bm{\Gamma}_k\},\{\bm{\Gamma}^{(t)}_k\})=\sum_{k=1}^K \lambda_k D_h(\bm{\Gamma}_k,\bm{\Gamma}^{(t)}_k)$, i.e., the IPOT-WB case, with very similar proof. This is because the latter is essentially just the weighted version of the former. 

Before proving both theorems, we propose several fundamental lemmas.
The first Lemma is the fundamental descent lemma, which is popularly used to analysis the convergence result of first-order methods.
\begin{lemma}
(Descent Lemma) Consider a closed proper convex function $f:X\rightarrow (-\infty,\infty]$ and for any $x\in X$ and $\beta^{(t)} > 0$, we have:
\begin{equation}\label{Descent_Lemma}
f(x^{(t+1)})\le f(x) + \beta^{(t)} \left[ D_h (x, x^{(t)}) - D_h (x,x^{(t+1)}) - D_h (x^{(t+1)},x^{(t)}) \right], \quad \forall x\in X.
\end{equation}
\end{lemma}
\begin{proof}
The optimality condition of (\ref{Bregman-iteration}) can be written as
$$
\left( x - x^{(t+1)} \right)^T \left[ \nabla f(x^{(t+1)}) + \beta^{(t)} \left( \nabla h (x^{(t+1)}) - \nabla h (x^{(t)}) \right) \right]\ge 0,\quad \forall x\in X.
$$Then with the convexity of $f$, we obtain
\begin{equation}\label{Basic-inequality-1}
f(x) - f(x^{(t+1)}) + \beta^{(t)} \left( x - x^{(t+1)} \right)^T \left( \nabla h (x^{(t+1)}) - \nabla h (x^{(t)}) \right) \ge 0.
\end{equation}With (\ref{three_point_lemma}) it follows that
$$
\left( x - x^{(t+1)} \right)^T \left( \nabla h (x^{(t+1)}) - \nabla h (x^{(t)}) \right) = D_h (x, x^{(t)}) - D_h (x,x^{(t+1)}) - D_h (x^{(t+1)},x^{(t)}).
$$Substitute the above equation into (\ref{Basic-inequality-1}), we have
$$
f(x^{(t+1)})\le f(x) + \beta^{(t)} \left[ D_h (x, x^{(t)}) - D_h (x,x^{(t+1)}) - D_h (x^{(t+1)},x^{(t)}) \right], \quad \forall x\in X.
$$
\end{proof}

Next, we prove the convergence result in {\bf{Theorem 5.1}}.

\medskip
\noindent {\bf{Theorem 5.1}} {\em{
Let $\{x^{(t)}\}$ be the sequence generated by the general Bregman proximal point algorithm with iteration (\ref{Bregman-iteration}) where $f$ is assumed to be continuous and convex.
Further assume that $f^* = \min f(x) > -\infty$.
Then we have that $\{f(x^{(t)})\}$ is non-increasing, and $f(x^{(t)})\rightarrow f^*$. Further assume there exists $\eta$, s.t.
\begin{equation}
\label{assume1}
f^* +\eta d(x) \leq f(x), \quad \forall x \in X,
\end{equation}
The algorithm has linear convergence.
}}
\begin{proof}
\begin{enumerate}
\item First, we prove the sufficient decrease property:
\begin{equation}\label{Sufficient-Decrease}
f(x^{(t+1)}) \le f(x^{(t)}) - \beta^{(t)}(1 + \alpha(h)) D_h (x^{(t+1)},x^{(t)}).
\end{equation}
Let $x=x^{(t)}$ in (\ref{Descent_Lemma}), we obtain
\begin{eqnarray}
f(x^{(t+1)})&\le& f(x^{(t)}) - \beta^{(t)} \left[ D_h(x^{(t)},x^{(t+1)}) + D_h (x^{(t+1)},x^{(t)}) \right]\nonumber\\
&\le& f(x^{(t)}) - \beta^{(t)}(1 + \alpha(h)) D_h (x^{(t+1)},x^{(t)}).\nonumber
\end{eqnarray}
With the sufficient decrease property, it is obvious that $\{f(x^{(t)})\}$ is non-decreasing.
\item Summing (\ref{Sufficient-Decrease}) from $i = 0$ to $i=t-1$ and for simplicity assuming $\beta^{(t)} = \beta$, we have
$$
\sum_{i=0}^{k-1} \left[ \frac{1}{\beta^{(t)}} \left( f(x^{(i+1)}) - f(x^{(i)}) \right) \right] \le - \left[ 1+\alpha(h) \right]\sum_{i=0}^{k-1} D_h (x^{(i+1)},x^{(i)})$$
$$\Rightarrow \qquad \sum_{i=0}^{\infty} D_h (x^{(i+1)},x^{(i)}) < \frac{1}{\beta(1+\alpha(h))} f(x^{(0)}) < \infty,
$$which indicates that $D_h (x^{(i+1)},x^{(i)}) \rightarrow 0$.
Then summing (\ref{Descent_Lemma}) from $i=0$ to $i=t-1$, we have
$$
k\left(f(x^{(t)}) - f(x)\right) \le \sum_{i=0}^{t-1} \left( f(x^{(i+1)}) - f(x) \right) \le \beta D_h (x,x^{(0)}) < \infty,\quad \forall x\in X.
$$Let $t\rightarrow \infty$, we have $\lim_{t\rightarrow \infty} f(x^{(t)}) \le f(x)$ for every $x$, as a result we have $\lim_{k\rightarrow \infty} f(x^{(t)}) = f^*$.
\item Finally, we prove the convergence rate is linear. Assume $x^* = \argmin_{x} f(x)$ is the unique optimal solution. Denote $d(x)= D_h(x^*,x)$. Let also $\beta^{(t)} = \beta$, we will prove
\begin{equation}
\label{rate}
 \frac{d(x^{(t+1)})}{d(x^{(t)})} \leq \frac{1}{1+\frac{\eta}{\beta}}
\end{equation}

Replace $x$ with $x^*$ in inequality (\ref{Descent_Lemma}), we have
\begin{equation}
\label{inqua1}
f(x^{(t+1)})\le f^* + \beta \left[ d (x^{(t)}) - d(x^{(t+1)}) - D_h (x^{(t+1)},x^{(t)}) \right].
\end{equation}
Using assumption \ref{assume1}, we have
\begin{equation}
\label{inqua2}
f^*+\eta d( x^{(t+1)})\leq f(x^{(t+1)})
\end{equation}
Sum \ref{inqua1} and \ref{inqua2} up, we have
\begin{align*}
\frac{\eta}{\beta}  d( x^{(t+1)}) & \leq  d (x^{(t)}) - d(x^{(t+1)}) - D_h (x^{(t+1)},x^{(t)}) \\
& \leq  d (x^{(t)}) - d(x^{(t+1)})
\end{align*}
Therefore,
\[
\frac{d(x^{(t+1)})}{d(x^{(t)})} \leq \frac{1}{1+\frac{\eta}{\beta}}
\]
Therefore, we have a linear convergence in Bregman distance sense. 
\end{enumerate}

\end{proof}

Assumption (\ref{assume1}) does not always hold when $f$ is linear. In our specific case, $x$ is bounded in $[0,1]^{m\times n}$. More rigorously, we can prove the following lemma.

\begin{lemma}
Assume $\mathcal{X}$ is a bounded polyhedron, $x^*$ is unique, $d(x)$ is an arbitrary nonnegative convex function with $d(x*)=0$. If $f$ is linear, then there exist $\eta$, s.t. $f^* + \eta d(x) \leq f(x)$.

\end{lemma}
\begin{proof}
Since $\mathcal{X}$ is a bounded polyhedron, any $x\in \mathcal{X}$ can be expressed as $x=\sum_{i=0}^n \lambda_i e_i$, where $e_i$ is the vertices of $\mathcal{X}$, $n$ is finite, and $\sum \lambda_i=1$. Also $f$ is linear, so $f(x)=\sum_{i=0}^n \lambda_i f(e_i)$

Since $f$ is linear, $\mathcal{X}$ is polyhedral and $x^*$ is unique, $x^*$ is a vertex of $X$. Denote $e_0=x^*$.

Denote $\delta=\min_{i>0} f(e_i)-f^*$, then $\delta>0$, or else $x^*$ is not unique. Denote $d_{max}=\max_{i>0} d(e_i)$. Take
\[
\eta=\delta/d_{max},
\]
we have
\begin{align*}
	f^*+\eta d(x)&= f^*+\eta d(\sum_{i=0}^n \lambda_i e_i)\\
	\text{ (Jensen's Inequality)}& \leq f^*+\eta\sum_{i=0}^n \lambda_i d(e_i)\\
	(d(e_0)=0)\quad& \leq f^*+(1-\lambda_0)\eta d_{max}\\
	& = f^*+(1-\lambda_0)\delta\\
	& = \sum_{i=1}^n \lambda_i (f^*+\delta)+\lambda_0 f^*\\
	& \leq \sum_{i=0}^n \lambda_i f(e_i)\\
	& = f(x).
\end{align*}
\end{proof}

For more general cases, if $x^*$ is not unique, we can divide the vertices as “optimal vertices” and the rest vertices, instead of $e_0$ and the rest as above, the conclusion can be proved analogously. Furthermore, if $\mathcal{X}$ is not a polyhedron, as long as $\mathcal{X}$ is bounded, we can always prove the conclusion in a polyhedron $\mathcal{A}$ s.t. $\mathcal{X}\in \mathcal{A}$ and $x*$ is also the optimal solution of $\min_{x \in \mathcal{A}} f(x)$. Proof of more general cases can be found in \citet{robinson1981some} (This paper points out some fairly strong continuity properties that polyhedral multifunctions satisfy).

Inequality (\ref{rate}) shows how the convergence rate is linked to $\beta$. This is the reason we claim in Section 4.1 that a smaller $\beta$ would lead to quicker convergence in exact case.

From above, we showed that the general Bregman proximal point algorithm with constant step size can guarantee convergence to the optimal solution $f^*$, and has linear convergence rate with some assumptions. Further, we prove the convergence result for the general Bregman proximal point algorithm with inexact scheme in {\bf{Theorem 5.2}}.

\medskip
\noindent {\bf{Theorem 5.2}} {\em{Let $\{ x^{(t)} \}$ be the sequence generated by the general Bregman proximal point algorithm with inexact scheme (i.e., finite number of inner iterations are employed).
Define an error sequence $\{ e^{(t)} \}$ where
\begin{equation}\label{Inexact-scheme}
e^{(t+1)} \in \beta^{(t)} \left[ \nabla f(x^{(t+1)} ) +  \partial \iota_{X} (x^{(t+1)} ) \right] +  \left[ \nabla h(x^{(t+1)} ) - \nabla h(x^{(t)} ) \right],
\end{equation}where $\iota_{X}$ is the indicator function of set $X$. If the sequence $\{ e^{(t)} \}$ satisfies $\sum_{k=1}^{\infty} \| e^{(t)} \| < \infty$ and $\sum_{k=1}^{\infty} \langle e^{(t)}, x^{(t)} \rangle$ exists and is finite, then $\{ x^{(t)} \}$ converges to $x^{\infty}$ with $f(x^{\infty}) = f^*$.
If the sequence $\{ e^{(t)} \}$ satisfies that exist $\rho\in (0,1)$ such that$ \| e^{(t)} \| \le \rho^t$, $ \langle e^{(t)}, x^{(t)} \rangle \le \rho^t$ and with assumption (\ref{assume1}), then $\{ x^{(t)} \}$ converges linearly.}}

\medskip
\noindent {\bf{Remark}}: If exact minimization is guaranteed in each iteration, the sequence $\{ x^{(t)} \}$ will satisfy that 
$$
0 \in  \beta^{(t)} \left[ \nabla f(x^{(t+1)} ) +  \partial \iota_{X} (x^{(t+1)} )\right] + \frac{1}{\beta^{(t)}} \left[ \nabla h(x^{(t+1)} ) - \nabla h(x^{(t)} ) \right].
$$As a result, with enough inner iteration, the guaranteed $e^{(t)}$ will goes to zero.

\medskip
\begin{proof}
This theorem is extended from \cite[Theorem 1]{eckstein1998approximate}, and we propose a brief proof here.
The proof contains the following four steps:
\begin{enumerate}
\item We have for all $k\ge 0$, through the three point lemma
$$
D_h (x,x^{(t+1)}) = D_h (x,x^{(t)}) - D_h (x^{(t+1)},x^{(t)}) - \langle \nabla h(x^{(t)}) - \nabla h(x^{(t+1)}), x^{(t+1)}-x \rangle,
$$which indicates
$$
D_h (x,x^{(t+1)}) = D_h (x,x^{(t)}) - D_h (x^{(t+1)},x^{(t)}) - \langle \nabla h (x^{(t)}) - \nabla h (x^{(t+1)}) + e^{(t+1)}, x^{(t+1)} - x \rangle + \langle e^{(t+1)}, x^{(t+1)} - x \rangle.
$$
Since $ \frac{1}{\beta^{(t)}}  \left[ e^{(t+1)} + \nabla h(x^{(t)} ) - \nabla h(x^{(t+1)} ) \right] \in \nabla f(x^{(t+1)} ) +  \partial \iota_{X} (x^{(t+1)} )$ and $0\in  \nabla f(x^{*}) +  \partial \iota_{X} (x^{*})$ if $x^*$ be the optimal solution, we have
\begin{eqnarray}
&&\langle \nabla h (x^{(t)}) - \nabla h (x^{(t+1)}) + e^{(t+1)}, x^{(t+1)} - x^{*} \rangle \nonumber\\
&=&\beta^{(t)}  \langle \left[ \frac{1}{\beta^{(t)}}  \left( \nabla h (x^{(t)}) - \nabla h (x^{(t+1)}) + e^{(t+1)} \right) \right] - 0, x^{(t+1)} - x^{*} \rangle \ge 0,\nonumber
\end{eqnarray}because $\nabla f + \partial \iota_{X}$ is monotone ($f+ \iota_{X}$ is convex).
Further we have
$$
D_h (x^*,x^{(t+1)}) \le D_h (x^*,x^{(t)}) - D_h (x^{(t+1)},x^{(t)}) +  \langle e^{(t+1)}, x^{(t+1)} - x^* \rangle.
$$
\item Summing the above inequality from $i=0$ to $i=t-1$, we have
$$
D_h (x^*,x^{(t)}) \le D_h (x^*,x^{(0)}) - \sum_{i=0}^{t-1} D_h (x^{(i+1)},x^{(i)}) +  \sum_{i=0}^{t-1} \langle e^{(i+1)}, x^{(i+1)} - x^* \rangle.
$$Since $\sum_{t=1}^{\infty} \| e^{(t)} \| < \infty$ and $\sum_{t=1}^{\infty} \langle e^{(t)}, x^{(t)} \rangle$ exists and is finite, we guarantee that
$$
\bar{E}(x^*) = \sup_{t\ge 0} \left\{ \sum_{i=0}^{t-1} \langle e^{(i+1)}, x^{(i+1)} - x^* \rangle \right\} < \infty.
$$Together with $D_h (x^{(i+1)},x^{(i)}) > 0$, we have
$$
D_h (x^*,x^{(t)}) \le D_h (x^*,x^{(0)}) + \bar{E}(x^*) < \infty,
$$which indicates
$$
0 \le \sum_{i=0}^{\infty} D_h (x^{(i+1)},x^{(i)}) < D_h (x^*,x^{(0)}) + \bar{E}(x^*) < \infty,
$$and hence $D_h (x^{(i+1)},x^{(i)}) \rightarrow 0$.
\item Based on the above two items, we know that the sequence $\{ x^{(t)} \}$ must be bounded and has at least one limit point $x^{\infty}$.
The most delicate part of the proof is to establish that $0\in \nabla f(x^{\infty}) + \partial \iota_{X} (x^{\infty})$.
Let $T = \nabla f + \partial \iota_{X}$, then $T$ denotes the subdifferential mapping of a closed proper convex function $f+ \iota_{X}$ ($f$ is a linear function and $X$ is a closed convex set).
Let $\{ t_j \}$ be the sub-sequence such that $x^{t_j}\rightarrow x^{\infty}$.
Because $x^{t_j}\in X$ and $X$ is a closed convex set, we know $x^{\infty}\in X$.
We know that $D_h (x^*,x^{(t+1)}) \le D_h (x^*,x^{(t)}) + \langle e^{(t+1)},x^{(t+1)} - x^* \rangle$ and $\sum_{k=0}^{\infty} \langle e^{(t+1)},x^{(t+1)} - x^* \rangle$ exists and is finite. From \cite[Section 2.2]{polyak1987introduction}, we guarantee that $\{ D_h (x^*,x^{(t)}) \}$ converges to $0 \le d(x^*) < \infty$.
Define $y^{(t+1)} := \lambda_k \left( \nabla h (x^{(t)}) - \nabla h (x^{(t+1)}) + e^{(t+1)} \right) $, we have
$$
\lambda_k \langle y^{(t+1)}, x^{(t+1)} - x^* \rangle = D_h (x^*,x^{(t)}) - D_h(x^*,x^{(t+1)}) - D_h (x^{(t+1)},x^{(t)}) + \langle e^{(t+1)}, x^{(t+1)} - x^* \rangle.
$$By taking the limit of both sides and $\lambda_k = \lambda > 0$, we obtain that
$$
\langle y^{(t+1)}, x^{(t+1)} - x^* \rangle \rightarrow 0.
$$For the reason that $y^{k_j + 1}$ is a subgradient of $f+\iota_{X}$ at $x^{k_j + 1}$, we have
$$
f(x^*) \ge f(x^{k_j +1 }) + \langle y^{k_j +1}, x^* - x^{k_j +1} \rangle, \quad x^*\in X, x^{k_j +1}\in X.
$$Further let $j\rightarrow \infty$ and using $f$ is lower semicontinuous, $\langle y^{(t+1)}, x^{(t+1)} - x^* \rangle \rightarrow 0$, we obtain
$$
f(x^{*})\ge f(x^{\infty}),\quad x^{\infty}\in X
$$which implies that $0\in \nabla f(x^{\infty}) + \iota_{X} (x^{\infty})$.
\item Recall the inexact scheme (\ref{Inexact-scheme}), we can equivalently guarantee that
$$
(x - x^{(t+1)})^T \left\{ \beta^{(t)} \nabla f(x^{(t+1)}) + \left[ \nabla h(x^{(t+1)}) - \nabla h(x^{(t)}) \right] - e^{(t+1)} \right\}\ge 0,\ \forall x\in X.
$$Together the convexity of $f$ and the three point lemma, we obtain
$$
f(x^{(t+1)}) \le f(x) + \frac{1}{\beta^{(t)}} \left[ D_{h} (x,x^{(t)}) - D_{h} (x,x^{(t+1)}) - D_{h} (x^{(t+1)},x^{(t)}) - (x - x^{(t+1)})^T e^{(t+1)} \right].
$$Let $x = x^*$ in the above inequality and recall the assumption (\ref{assume1}), {\em{i.e.}},
$$
f(x) - f(x^*) \ge \eta d(x),
$$we have with $\beta^{(t)} = \beta$
\begin{eqnarray}
\eta d(x^{(t+1)}) &\le& \frac{1}{\beta} \left[ d(x^{(t)}) - d(x^{(t+1)}) \right] + \frac{1}{\beta} \left( (x^{(t+1)} - x^*)^T e^{(t+1)} \right)\nonumber\\
&\le& \frac{1}{\beta} \left[ d(x^{(t)}) - d(x^{(t+1)}) \right] + \frac{1}{\beta} \left( C\|e^{(t+1)}\| + \langle x^{(t+1)},e^{(t+1)} \rangle \right),\nonumber
\end{eqnarray}where $C:=\sup_{x\in X^*} \{ \|x\| \}$. The second inequality is obtained through triangle inequality. Then
$$
d^{(t+1)} \le \mu d^{(t)} + \mu \left( C\|e^{(t+1)}\| + \langle x^{(t+1)},e^{(t+1)} \rangle \right),
$$where $\mu = \frac{1}{1 + \beta \eta} < 1$.
With our assumptions and according to Theorem 2 and Corollary 2 in \citet{So2017non}, we guarantee the generated sequence converges linearly in the order of ${\cal{O}}\left( c^t \right)$, where $c = \sqrt{\frac{1+ \max\{ \mu,\rho \}}{2}}\in (0,1)$.
\end{enumerate}

Based on the above four items, we guarantee the convergence results in this theorem.
\end{proof}

\end{document}